\documentclass{article}
\usepackage[utf8]{inputenc}
\usepackage{amssymb}
\usepackage{amsmath}
\usepackage{amsthm}
\usepackage{dsfont}
\usepackage{bbm}
\usepackage{color}
\usepackage[margin=1in]{geometry}
\usepackage{mathrsfs}
\usepackage{mathtools}
\usepackage{graphicx}
\usepackage{stackengine}
\usepackage{soul}
\usepackage{todonotes}
\usepackage{enumitem}
\usepackage{subcaption}
\usepackage{relsize}
\usepackage{array}
\usepackage{scalerel}
\usepackage{cancel}
\usepackage{enumitem}
\usepackage{centernot}
\usepackage[
  backend=bibtex, 
  style=numeric, 
  sorting=nyt, 
  url=false, 
  doi=false, 
  isbn=false, 
  maxcitenames=10,  
  mincitenames=10,  
  maxbibnames=10,   
  minbibnames=10    
]{biblatex}

\AtEveryBibitem{%
  \clearfield{note} 
  \clearfield{eprint} 
  \clearfield{archivePrefix} 
}
\usepackage{pgfplots}
\pgfplotsset{compat=1.16}

\theoremstyle{plain}
\newtheorem{theorem}{Theorem}[section] 

\usetikzlibrary{calc}
\newtheorem{defn}[theorem]{Definition} 
\newtheorem{exmp}[theorem]{Example}
\newtheorem{lemma}[theorem]{Lemma}
\newtheorem{remark}[theorem]{Remark}
\newtheorem{corollary}[theorem]{Corollary}
\newtheorem{prop}[theorem]{Proposition}
\DeclareMathOperator\supp{supp}
\newcommand{\tpi}{\Tilde{\pi}}
\newcommand{\tnu}{\Tilde{\nu}}
\newcommand{\R}{\mathbb{R}}

\newcommand{\B}{\mathcal{B}}

\newcommand{\I}{\mathcal{I}}
\newcommand{\J}{\mathcal{J}}

\renewcommand{\O}{\mathcal{O}}
\newcommand{\Pp}{\mathscr{P}}

\newcommand{\X}{\mathcal{X}}

\newcommand{\Y}{\mathcal{Y}}

\newcommand{\La}{(\bf{L})}

\newcommand{\e}{\varepsilon}
\newcommand{\spt}{\mathrm{Spt}\,}
\newcommand{\dpi}{\rotatebox[origin=c]{6}{\topinset{\rotatebox[origin=c]{180}{$\gamma$}}{$\gamma$}{-1.7pt}{1.55pt}\hspace{0.04cm}}}

\renewcommand{\iint}{\int\hspace{-0.285cm}\int}

\definecolor{darkred}{rgb}{0.6,0.1,0.1}
\definecolor{darkgreen}{rgb}{0.1,0.6,0.1}
\definecolor{darkblue}{rgb}{0.1,0.1,0.6}
\numberwithin{equation}{section}

\def\rm#1{\textbf{\textcolor{darkgreen}{#1}}}

\def\ap#1{\textbf{\textcolor{red}{#1}}}

\counterwithin{equation}{section}

\title{On Probabilistic Embeddings in Optimal Dimension Reduction}
\author{Ryan Murray \& Adam Pickarski\thanks{Department of Mathematics, North Carolina State University}} 
\date{\today}

\begin{document}
\maketitle
\begin{abstract}
     Dimension reduction algorithms are a crucial part of many data science pipelines, including data exploration, feature creation and selection, and denoising. Despite their wide utilization, many non-linear dimension reduction algorithms are poorly understood from a theoretical perspective. In this work we consider a generalized version of multidimensional scaling, which is posed as an optimization problem in which a mapping from a high-dimensional feature space to a lower-dimensional embedding space seeks to preserve either inner products or norms of the distribution in feature space, and which encompasses many commonly used dimension reduction algorithms. We analytically investigate the variational properties of this problem, leading to the following insights: 1) Solutions found using standard particle descent methods may lead to non-deterministic embeddings, 2) A relaxed or probabilistic formulation of the problem admits solutions with easily interpretable necessary conditions, 3) The globally optimal solutions to the relaxed problem actually must give a deterministic embedding. This progression of results mirrors the classical development of optimal transportation, and in a case relating to the Gromov-Wasserstein distance actually gives explicit insight into the structure of the optimal embeddings, which are parametrically determined and discontinuous. Finally, we illustrate that a standard computational implementation of this task does not learn deterministic embeddings, which means that it learns sub-optimal mappings, and that the embeddings learned in that context have highly misleading clustering structure, underscoring the delicate nature of solving this problem computationally.
\end{abstract}

\section{Introduction}
A central task in data science is to find efficient representations of high-dimensional data. One form of this task is known as \emph{dimension reduction}, in which one seeks to construct a mapping from a high-dimensional space to a low-dimensional space which approximately preserves features of an input distribution. Dimension reduction serves many purposes: it aids in data visualization and exploration, feature construction, and denoising. Dimension reduction is often stated in terms of some optimization problem, and naturally the properties and computational tractability are dependent upon the particular dimension reduction objective.

In this work, we consider dimension reduction problems corresponding to optimization problems of the form
\begin{displaymath}
  \min_T \sum_{ij} c(X_i,X_j,T(X_i),T(X_j)),
\end{displaymath}
where we are considering the $X_i \in \R^d$ to be data points in a high-dimensional feature space, and $T:\R^d \to \R^m$ represents a mapping, or embedding, into a lower dimensional space. A simple mnemonic here is that `$d$' is for ``data'' and `$m$' is for ``embedding''. In order to accommodate both finite data sets and large sample or population limits, we consider a generalized problem of the form
\begin{equation}\label{eqn:MDS-general}
    \J(T):= \iint c(x,x',T(x),T(x')) \mu(dx)\mu(dx'),
\end{equation}
where we will assume that $\mu \in \Pp(\R^d)$, the space of probability measures on $\R^d$. Throughout this work we make very few assumptions upon $\mu$: it could be supported on a discrete point cloud, a low dimensional manifold, or a continuous probability distribution. We call this problem the \emph{second-order dimension reduction problem}, where by second-order we mean that the objective function considers pairwise, or second-order, interactions between points. This problem encompasses many common dimension reduction problems, see Section \ref{sec:Normed_costs} for examples. Variants of this general problem have also been considered under the heading of multi-dimensional scaling and quadratic assignment problems. While not all dimension reduction algorithms can be written in this second-order form, such algorithms generally serve as building blocks for many commonly used methods, see Section \ref{sec:related-work} for more discussion.

Perhaps the simplest version of this form of problem is \emph{Classical Multidimensional Scaling} (cMDS), which, in the discrete setting and with $\sum_i X_i = \sum_i T(X_i) = 0$, seeks to minimize the objective function
\begin{equation}
   \min_{\{Y_\ell\}_{\ell=1}^n} \sum_{ij} (\langle X_i,X_j \rangle - \langle Y_i-\mathbb{E}[Y],Y_j -\mathbb{E}[Y]\rangle )^2.
\end{equation}
Alternatively, this can be written, again assuming that $\mathbb{E}(X) = \mathbb{E}[T(X)] = 0$, 
\begin{equation}
    \min_{T: \R^d \to \R^m} \iint (\langle x,x' \rangle - \langle T(x),T(x')\rangle )^2 \mu(dx)\mu(dx').
\end{equation}
In both versions of this problem the minimizer is known to be a linear mapping, implying that the minimizer is parametrically determined and smooth. Furthermore, this minimizer can be described as the projection onto the $m$-dimensions of greatest variance of $\mu$, and is equivalent to PCA. This approach to dimension reduction is prevalent in many contexts.

However, in some settings linear embeddings are too restrictive to capture important structures in data. For this reason a host of different cost functions have been proposed for dimension reduction, each emphasizing distinct priorities. In many contexts these algorithms are able to flexibly capture important features of high-dimensional distributions inaccessible to linear embeddings, but this flexibility comes at a price: non-linear dimension reduction problems generally can only be resolved via optimization routines, and their solutions do not admit transparent parametric representation formulas. As such, in many cases theoretical properties of the solutions to these problems are poorly understood. In particular, in the setting where $\mu$ is a continuum distribution, i.e. the large data or population limit, and when $c$ is non-convex, it is not clear whether the problem \eqref{eqn:MDS-general} even admits a minimizer. We will discuss negative results in the mathematical literature along these lines in Section \ref{sec:existence}, but in simplified terms for non-convex energies it is possible for approximate minimizers to converge towards a limit which is not a function. While the issue of existence is often straightforward in the finite data setting, the lack of a meaningful population limit raises significant issues for optimization and interpretability of minimizers: we highlight this issue with a simple numerical experiment in Example \ref{exmp:numerical}.



Similar issues were long-standing in the theory of optimal transportation, and our approach in this paper mirrors that literature. In that context, the Monge formulation of optimal transportation seeks to minimize
\begin{equation}
    \min_{T:\R^d \to \R^d, T_\sharp \mu = \nu} \int c(x,T(x)) \mu(dx).
\end{equation}
Here $\nu$ is an output distribution and $T_\sharp$ denotes the pushforward measure. Demonstrating that Monge's problem has a solution was a major open problem for many years, and while the dimension reduction problem notably lacks the output distribution constraint, the overall lack of convexity with respect to $T$ still engenders a similar type of issue.

The technical solution to this issue in optimal transportation is to instead consider a relaxed version of the problem, namely
\begin{equation}
\label{OTprob}
    \min_{\pi \in \Pi(\mu,\nu)} \int c(x,y) \pi(dx\: dy),
\end{equation}
where $\Pi(\mu,\nu)$ is the set of probability distributions on $\R^d \times \R^d$ with marginals $\mu,\nu$: such probability measures in $\Pi(\mu,\nu)$ are called transportation plans and are multi-valued generalizations of the transportation map $T$ sought for in the Monge problem. In short, this formulation relaxes the requirement that $x$ is mapped ``deterministically'' to a single $T(x)$, and instead permits a single $x$ to be mapped probabilistically to multiple outputs. Demonstrating that this problem has a solution using ``soft'' analytical methods is straightforward. Subsequently, one can establish structural properties of such relaxed solutions. Using tools such as cyclical monotonicity and convex analysis, one can demonstrate that under mild assumptions minimizers of \eqref{OTprob} are actually induced by a mapping, which means that the original Monge problem possesses a solution. We can similarly pose a relaxed version of the MDS problem by seeking to minimize
\begin{equation} \label{DR}
   \J(\pi):=\left\{\iint c(x,x',y,y') \pi(dx\:dy) \pi(dx'dy'),\,\pi \in \Pi(\mu) \right\}
\end{equation}
where we let $\Pi(\mu)$ denote the set of distributions on $\R^d \times \R^m$ which have marginal $\mu$ in the first $d$ coordinates and refer to this as the set of \emph{embedding plans}. Throughout this article, we often use the notation $\X=\R^d$ and $\Y=\R^m$ to avoid confusion about which space we are embedding to.

The optimal transportation problem is inherently one of linear programming, whereas the dimension reduction problem is more aptly seen as a non-convex quadratic program (see Example \ref{DR_is_non_convex}). We mention that there is a quadratic programming variant of optimal transportation. In particular the Gromov-Wasserstein metric\footnote{The Gromov-Wasserstein metric generally is defined between two metric measure spaces, but we restrict our attention here to distributions on two different Euclidean spaces due to the connection with dimension reduction.} between distributions $\mu,\nu$, supported respectively on $\R^d$ and $\R^m$, is defined by the minimization problem \cite{article}
\begin{equation}
\label{GW_gen}
    d_{GW_{p,q}}(\mu,\nu)^{p} = \min_{\pi \in \Pi(\mu,\nu)} \iint \Big| |x-x'|^q-|y-y'|^q\Big|^p \pi(dx\: dy) \pi(dx'dy').
\end{equation}
In the Gromov-Wasserstein problem one generally has two marginal constraints, whereas in the dimension reduction problem there is only a single marginal constraint. As such, we can cast the dimension reduction problem as a projection problem in the Gromov-Wasserstein space: namely if we let $c(x,x',y,y') = \Big||x-x'|^q-|y-y'|^q\Big|^p $ then we have that $\min_{\pi\in\Pi(\mu)} \J(\pi) = \min_{\nu} d_{GW_{p,q}}(\mu,\nu)^{p}$.

The question of whether minimizers of the  Gromov-Wasserstein problem are always induced by transportation maps has recently be studied in \cite{article2,vayer2020contribution}, and their conclusion is that a deterministic minimizer (i.e. a minimizer of \eqref{GW_gen} which is supported on the graph of a function) is guaranteed to exist. Their work however does not establish whether minimizers are necessarily deterministic, a question which is relevant to solutions found via optimization routines. 

A natural question in the context of dimension reduction is whether optimal plans are necessarily maps, or in other words whether solutions to the relaxed problem \eqref{DR} are always solutions of the original problem \eqref{eqn:MDS-general}. The following example demonstrates that for numerically constructed local minimizers, this is not always the case.

\begin{figure}[h]
    \centering
    \begin{subfigure}[b]{0.48\textwidth}
         \centering
         \includegraphics[width=\textwidth]{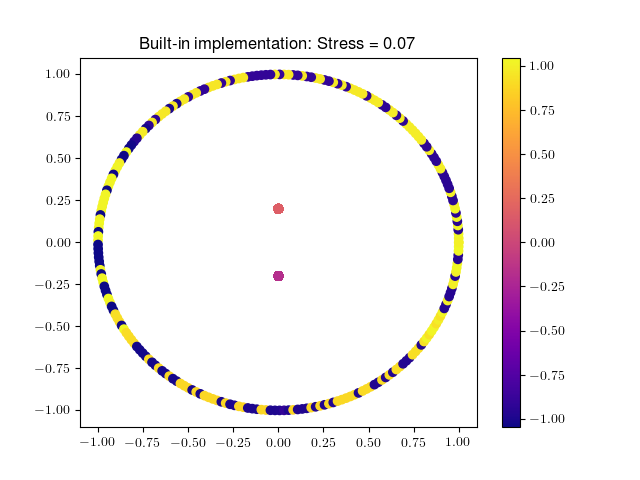}
         \caption{Scikit-Learn Embedding}
         \label{fig:sk-learn-embedding}
     \end{subfigure}
     \hfill
     \begin{subfigure}[b]{.48\textwidth}
         \centering
         \includegraphics[width=\textwidth]{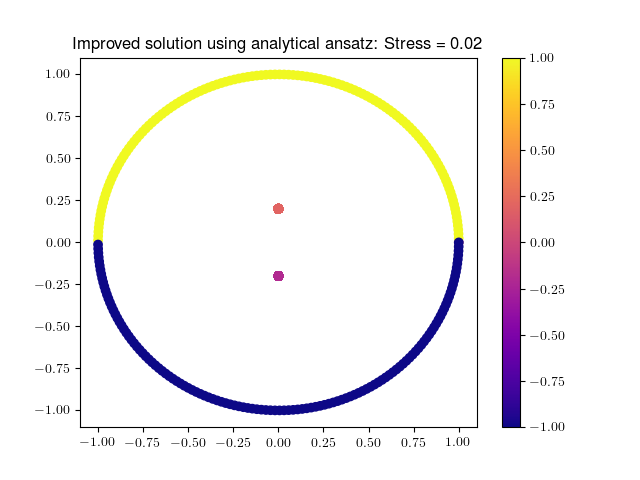}
         \caption{Globally optimal embedding?}
         \label{fig:optimal-embedding}
     \end{subfigure}
    \caption{An example where standard algorithms find locally optimal solutions which are not maps. Here the position of the points represents the original features in $\X = \R^2$, whereas the color represents the learned embedding in $\Y = \R$. The first graph shows the embedding learned by the implementation of metric MDS in Scikit-learn, and the second graph shows the embedding Scikit learn finds if given an analytically-motivated initial guess. The stress values, normalized by the the number of points squared, is also displayed, with a clear improvement in the second image. }
    \label{fig:comp-example}
\end{figure}

\begin{exmp}
\label{exmp:numerical}\footnote{The computation in this example was discovered in collaboration with Brian Swenson, and work about computational aspects of this problem is ongoing.}
\label{sklearn}
We consider the problem of embedding a particular point cloud in $\R^2$ into $\R$. The point cloud that we choose has 1,000 points placed at $(0,\pm .2)$, as well as 250 points placed randomly upon the unit circle. When we utilize the built-in algorithm for metric multidimensional scaling in Scikit-learn, the embedding which is found is very discontinuous: this is illustrated in Figure \ref{fig:sk-learn-embedding}. Indeed, changes around the boundary of the unit circle do not have a discernible pattern, and appears to be non-deterministic. The reason for this behavior is that due to the larger clusters near the origin, points on the unit circle are energetically favorable at either $\pm 1$, in the sense that both are local minimizers when other points are held fixed. These local minimizers are both nearly global minimzers as well, as the relative costs of being at either plus or minus one are comparable: this is due to the fact that the two larger clusters are relatively close together. The behavior of the solutions found indeed suggests that non-deterministic embedding plans can be local minimizers of the energy, at least if perturbations are only considered in the sense of small changes to particle positions. 

However, working by hand we would expect that the optimal embedding should be much more principled, and should map halves of the circle deterministically to different sides of the real line, according to the cluster they are closer to. Figure \ref{fig:optimal-embedding}, uses this ansatz to construct an initial guess for the same optimization routine in Scikit-learn. The learned embedding, while still having a jump discontinuity, is more interpretable and also obtains a significantly lower cost.
\end{exmp}

The previous example is, in the authors' opinion, rather arresting from the practical point of view. The embedding constructed by standard libraries has found four well-separated clusters, but two of the clusters were constructed by breaking up the unit sphere in a completely arbitrary fashion. Considering those two clusters as useful features or groups is clearly misleading at best.

This computational example also highlights potential mathematical challenges to proving the equivalence between the original and the relaxed dimension reduction problems, that is between problems \eqref{eqn:MDS-general} and \eqref{DR}. Indeed, the embedding learned by the standard implementation ought to be a local minimizer in some sense, suggesting that it may be possible to find local minimizers of \eqref{DR} which are not mappings.

This work aims to address these questions, in certain contexts, through the following contributions:

\begin{enumerate}
  \item (Proposition \ref{DR_lacks_weak_LSC}) We show that the dimension reduction energy \eqref{eqn:MDS-general} is not weakly lower semi-continuous in any $L^p$ space for many natural choices of $c$, meaning that existence of minimizers cannot be established using the direct method of the calculus of variations. In practice, this can lead to highly oscillatory (i.e. non-deterministic) solutions and poor local minima during gradient descent, as demonstrated in Example \ref{exmp:numerical}.
     \item (Theorems \ref{existence} $\&$ \ref{existence_NC}) Under appropriate conditions, we first show that for costs of the form $c(x,x',\langle y,y'\rangle)$ and $c(x,x',|y-y'|^2)$, the relaxed problem \eqref{DR} has a minimizer. This is mostly a consequence of standard arguments from the calculus of variations.
    \item (Theorem \ref{marginal_problem}) For the same class of costs, we demonstrate that any minimizer, $\pi$, of \eqref{DR} is essentially supported on the set
    \begin{equation}
        \{y : J_\pi(y|x) \text{ is minimized in } y\},\qquad J_\pi(y|x) := \int c(x,x',y,y')\pi(dx'dy').
    \end{equation}
    We call the problem of minimizing $J_\pi(y|x)$ the \emph{Marginal Problem}. This problem in many cases provides a significant constraint upon the form of $\pi$. The argument here relies upon a construction of localized perturbations inspired by needle perturbations from control theory.
  \item (Corollary \ref{corollary:IP_mins}) We show that certain costs, namely those for which $\langle y,y'\rangle \mapsto c(x,x',\langle y,y'\rangle)$ is convex, will only admit deterministic minimizers of \eqref{DR}: in the jargon of optimal transportation such solutions are maps. These solutions will furthermore have smoothness controlled by the differentiability of $c$.   
  \item (Theorem \ref{DP}) We show that for costs for which $|y-y'|^2\mapsto c(x,x,|y-y'|^2)$ has a unique minimum at $y=y'$ then minimizers of \eqref{DR} will necessarily be deterministic.
  \item (Examples \ref{qMDS_MP}, \ref{exmp:discontinuous}, $\&$ \ref{exmp:equivilance_classes}) We discuss in depth the example of a quartic cost in $|y-y'|$ stemming from Gromov-Wasserstein spaces, which is known to give non-linear embeddings. In that context we can additionally show that minimizers admit a parametric representation and have discontinuities along specific hyperplanes.
\end{enumerate}

These results have direct consequences for computational dimension reduction and their applications for practitioners, which we further discuss in Section \ref{sec:conclusion}.

The remainder of the work is organized as follows: in Section \ref{sec:related-work} we discuss literature from related fields, including various methods for dimension reduction and optimal transportation. In Sections \ref{sec:existence}, \ref{sec:marginal}, and \ref{sec:Normed_costs}, we prove the main results for generic costs, namely in Section \ref{sec:existence} we prove the existence of solutions to the relaxed problem \eqref{DR}, in Section \ref{sec:marginal} we demonstrate that the support of optimal plans is determined by the Marginal Problem and that similarity costs which are convex in the inner product necessitate deterministic minimizers, and in Section \ref{sec:Normed_costs} we describe how to obtain a similar result for normed squared costs. In Section \ref{sec:marginal} we investigate some finer properties of the Gromov-Wasserstein projection problem which also serves to motivate the theoretical considerations in Section \ref{sec:Normed_costs}. In Section \ref{sec:conclusion} we discuss ramifications of these results, as well as some further questions.

\subsection{Related Work}\label{sec:related-work}

Dimension reduction, and specifically Multidimensional Scaling (MDS), has a long history: we refer the reader to the books \cite{coxMDS,borg2005modern} for an in-depth classical statistical treatment of MDS. We mention here that MDS has extensions to a variety of settings, such as the setting where the original points belong to a metric space, or even where we only have access to a matrix of similarities or dissimilarities between our $x$'s. In certain applied fields, such as psychology \cite{kruskal1964}, MDS has been utilized extensively for group identification, and is cited in \cite{borg2005modern} as an important tool for data exploration. In the case of classical Multidimensional Scaling, which is equivalent to PCA, the explicit representation of solutions has facilitated many theoretical works, see for example \cite{10.1214/20-EJS1720} and the references therein. Several computational approaches have also been developed for speeding up the computation of MDS embeddings. Some references on the topic include \cite{JSSv031i03,yang2015majorization}.

On the other hand, in the last twenty years there has been extensive development of new dimension reduction techniques within the context of data science. A standard introductory reference for many of those types of algorithms is cite{james2013introduction} Chapter 14, and an in-depth comparison of various non-linear dimension reduction techniques can be found in \cite{van2009dimensionality}. These algorithms take a variety of approaches for preserving either global or local structure. Some notable examples include local linear embeddings, isomap, spectral embeddings, Sammon mapping, Multidimensional Scaling, and stochastic neighborhood embeddings \cite{doi:10.1126/science.290.5500.2323,Tenenbaum-Langford,6789755,1671271,kruskal1964,hinton2002stochastic}.

While the examples in this paper are fairly general, there are dimension reduction methods which go beyond our framework as they utilize locally adaptive kernels, for example tSNE, UMAP, or LLE \cite{van2008visualizing,mcinnes2018umap,doi:10.1126/science.290.5500.2323}. There has been recent interest in the mathematical community for identifying simplified models and techniques for understanding tSNE; see for example \cite{auffinger2023equilibriumdistributionstdistributedstochastic} which uses stochastic processes and random matrix theory techniques. There are also some mathematical works which seek to describe specific aspects of finding ``good'' solutions to SNE \cite{steinerberger-1}, in particular by studying early exaggeration techniques commonly used for training. We also remark that variants of the quartic example that we focus on in this work has previously been identified in the statistical learning literature as a particular scaling limit of tSNE \cite{hinton2002stochastic}.

This work has been significantly influenced by the development of the theory of optimal transportation, a good introduction to which can be found in \cite{villani2008optimal}. Recent works in the OT literature, such as multi-marginal transport \cite{pass2015multi} 
 and transport between spaces of unequal dimension \cite{nenna2020variational}, have also dealt with scenarios similar to ours, but in situations with linear dependence on $\pi$.

 There has also been a lot of interest recently in the Gromov-Wasserstein distance \cite{memoli2011gromov}, which provides a transportation-based metric between probability measures on two different metric spaces. Very recently multiple authors \cite{vayer2020contribution,article2} have studied the question of whether optimal plans in the Gromov Wasserstein problem are in fact realized by mappings. These works attempt to convert the Gromov-Wasserstein problem into an inhomogeneous linear (in $\pi$) problem, which then they tackle by using general optimal transportation theory: these works are able to show the existence of an optimal mapping. In particular, in \cite{article2} a Monge mapping is constructed as a solution to the GW problem, but the necessity of a deterministic solution is still an open question. Furthermore, earlier works such as \cite{vayer2020contribution} show that in the quartic setting, if a certain correlation matrix is non-degenerate then any optimal plan must be induced by a mapping. However, it is unclear how to directly prove that those correlations are in fact non-degenerate. Similarly, in \cite{arya2024gromovwassersteindistancespheres}, a Monge mapping was constructed in the special case between two spheres. Lastly, in \cite{titouan2019sliced} it was claimed that when $d=m=1$, that optimal solutions admit simple representations (as a monotonic map); however more recent work \cite{doi:10.1137/22M1497808} refuted this claim and provided a counterexample.
 
 It is important to note the connection between the Gromov-Wasserstein problem and quadratic assignment problems (QAP). In its original formulation \cite{6b66faf6-3f33-3aa9-8609-22b8cfa3f37a}, the quadratic assignment problem describes a variant of the optimal transport problem, wherein the function we minimize is of second degree in the unknown permutation matrix. A notable example of the QAP is the graph matching problem which matches the edges of two graphs in a meaningful way. This can rightly be viewed as a type of Gromov-Wasserstein problem.
 
 We also mention that there has also been a lot of recent work trying to find fast algorithms for GW problems, see for example \cite{peyre2016gromov,titouan2019sliced,LTGW}. The parametric form we derive for quartic MDS suggests that faster algorithms may also be available for the GW projection problem as well.

 Finally, there has been a vein of mathematical literature \cite{Pedregal1997,Elbau2011,Bellido2014,Foss2018} treating the minimization of energies of the form
 \begin{displaymath}
   \min_{u} \, \mathcal{I}(u),\qquad \mathcal{I}(u) := \iint \Phi(x_1,x_2,u(x_1),u(x_2)) \,dx_1\,dx_2.
 \end{displaymath}
 The main focus of these works has been to establish conditions which guarantee the existence of minimizers for energies of this type, by proving weak lower-semicontinuity in an appropriate topology. To our knowledge each of these results requires some form of convexity with respect to $\Phi$. Our work strongly contrasts that line of work, in that 1) we study forms of $\Phi$ with specific symmetries, 2) we demonstrate that our energies are \emph{not} weakly lower semicontinuous, and 3) we demonstrate that in spite of this lack of weak lower semicontinuity that there still exists minimizers of our original dimension reduction problem.

%
%

\section{Existence of Relaxed Solutions}\label{sec:existence}

In this section we consider the problem of existence of minimizers of \eqref{eqn:MDS-general} and \eqref{DR}. Along the way, we demonstrate that many of the standard techniques from the calculus of variations do not apply to the original problem of finding an embedding map as in \eqref{eqn:MDS-general}, namely the lack of weak lower semi-continuity. These theoretical observations directly compliment the phenomenon observed in Example \ref{sklearn}, and demonstrate the difficulty of proving properties of minimizers of the original problem \eqref{eqn:MDS-general}.

Convexity plays a crucial role in proving existence of minimizers for many variational problems. We begin by demonstrating, through a simple example, why convexity can fail in second-order dimension reduction problems.

\begin{exmp}\label{nonconvexqmds} We consider, as a running example throughout the paper, the quartic cost $c(x,x',y,y')=\big(|x-x'|^2-|y-y'|^2\big)^2$. Fix $\varepsilon>0$ and let $T\in C^1(\R^d;\R^m)$ be a Lipschitz function such that $\|DT\|_\infty \leq \sqrt{2-\varepsilon}$. We consider the effect of interpolating between $T(x)$ and $-T(x)$. Clearly the midpoint between these two maps is identically zero (we call this map the ``zero map'' through the paper), namly $\frac{1}{2}(T(x) - T(x)) \equiv 0$, and furthermore from the norm structure of the cost we immediately have that $\J(T)=\J(-T)$. Hence if $\J$ were midpoint convex, one would require $\J(T)\geq \J(0)$. However, we have\begin{align*}
\J(T)-\J(0)&=\iint |T(x)-T(x')|^4-2|x-x'|^2|T(x)-T(x')|^2\mu(dx)\mu(dx')\\
&\leq \iint (2-\varepsilon)|x-x'|^2|T(x)-T(x')|^2-2|x-x'|^2|T(x)-T(x')|^2\mu(dx)\mu(dx')\\
&=-\varepsilon\left(\iint |x-x'|^2|T(x)-T(x')|^2\mu(dx)\mu(dx')\right) \leq 0.
\end{align*}
If $\mu$ has a direction of non-zero variance, and $T$ is chosen to also vary in that direction, then this inequality is strict: one can find a linear mapping which achieves this goal. Hence $\J$ is not convex with respect to $T$.
\end{exmp}
It turns out that the previous observation, which primarily stems from the reflection symmetry of the quartic cost, extends to many second-order costs that have been previously considered for dimension reduction. In order to make the problem more concrete, we focus on two basic examples which encompass a broad family of practical situations. The first models the interactions of the embedded variables by an inner product, the second by a squared norm: these costs are known in the literature for multidimensional scaling as \emph{similarity} and \emph{dissimilarity} costs respectively. For concreteness we state these assumptions upon the structure more explicitly in Section \ref{sec:ex-relaxed}. However, with this distinction in place, we can now restate the non-convexity result above in more generality. Later we will provide suitable assumptions to identify the domain of definition for the dimension reduction problem.

\begin{prop}\label{DR_is_non_convex}
If the functionals\[
\J_{\mathbf{IP}}(T)=\iint c(x,x',\langle T(x),T(x')\rangle)\mu(dx)\mu(dx'),\qquad\J_{\mathbf{N}^2}(T)=\iint c(x,x',|T(x)-T(x')|^2)\mu(dx)\mu(dx')
\]are finite for functions in $L^p(\R^d;\R^m|\mu)$, and $T \equiv 0$ is not the global minimizer, then $\J_{\mathbf{IP}}\,\&\,\J_{\mathbf{N}^2} $ are neither convex nor concave on $L^p(\R^d;\R^m|\mu)$.
\end{prop}
\begin{proof}
  The proof follows exactly as in Example \ref{nonconvexqmds}: If $\J(T) = \J(-T) < \J(0)$ for some $T$ then $\J$ cannot be midpoint convex. Furthermore, $\J$ cannot be concave if it is non-constant and positive.

\end{proof}
As mentioned above, this lack of functional convexity will become a significant theoretical obstacle: this type of obstacle is well-known in the literature for the theory of the Calculus of Variations. In order to clarify this obstruction to a broader audience, we use the next section to provide a brief overview of this theory. A reader who is comfortable with all of these concepts can safely skip to Section \ref{sec:ex-relaxed}; the main result being Proposition \ref{DR_lacks_weak_LSC} which demonstrates that the dimension reduction energy cannot be a weakly lower semi-continuuous functional on $L^p(\R^d;\R^m|\mu)$.

\subsection{A detour into theory of variational problems}
We now describe the reason why existence of minimizers to \eqref{eqn:MDS-general} is a challenging problem. To put it concisely, the standard ``direct method'' from the calculus of variations does not apply due to the non-convexity of $c$. This occurs both due to the generic the lack of (strong) compactness in $L^p(\R^d;\R^m|\mu)$ and the failure of weak lower semi-continuity of $\J$ in the same space. We provide a number of standard examples to clarify these phenomena to a broader audience. 

The direct method of the calculus of variations seeks to generalize the extreme value theorem in finite dimensions to infinite dimensional optimization problems. It proves the existence of minimizers of a functional $\I: U \to \R$, where $U$ is an infinite-dimensional space, by combining the following assumptions:
\begin{enumerate}
  \item \textbf{Coercivity:} Given some set $B \subset U$ we have that $\I(B^c) > \inf_U \I$.
  \item \textbf{Compactness:} Under some topology $\tau$ we have that $B$ is sequentially compact.
  \item \textbf{Continuity:} Under that same topology, the functional $\I$ is sequentially lower semi-continuous.
\end{enumerate}
One then directly shows the existence of minimizers by taking the following steps: i) Construct a sequence of functions $u_n \in B$ so that $\lim \I(u_n) = \inf_U \I$, ii) After taking a subsequence, $u_n \to_\tau u^*$, and iii) Using the lower semi-continuity we have that $\I(u^*) \leq \liminf \I(u_n)$, implying that $u^*$ is a minimizer.

The main challenge in carrying out this approach is that if $U$ is an infinite-dimensional normed space and $B$ is some ball in that norm, then $B$ can never be compact under the same norm. As such, one needs to select a weaker topology that allows compactness. The price to pay is that in weaker topologies continuity of $\I$ is a stronger condition to verify.

In this section, we will primarily focus on $L^p$ type spaces, because for many notable examples we expect minimizers of our variational problem to fail to be continuous. To illustrate why this is the case, we begin with a toy problem demonstrating how non-convex functional optimization can have discontinuous minimizers.
\begin{exmp}[Double-well Potential]
\label{exmp:double_well}
Let $f(x,u)=\tfrac{1}{4}(u^2-1)^2-ux$ and define the  functional \begin{equation}\label{dub_well}
\I(u)=\int_{-1}^1 f(x,u(x))dx.
\end{equation}
In this simple case, one can directly show that the minimizer of this functional is given by $u^*(x)\in\arg\min f(x,\cdot)$ for every $x\in[-1,1]$. We display this function in Figure \ref{fig:multi-well-minimizer}, and the discontinuity at $x=0$ is apparent. This occurs because there are two distinct, well-separated, global minima at $x=0$. We notice that for $x\in\big[-\tfrac{1}{3^{3/2}}-\tfrac{1}{3^{1/2}},\tfrac{1}{3^{3/2}}+\tfrac{1}{3^{1/2}},\big]$, the function $f(x,\cdot)$ has 2 local minima, and that the global minima switches from one side to the other at $x=0$: this is illustrated in Figure \ref{fig:multi-well}.
\begin{figure}[h]
    \centering
    \begin{subfigure}[b]{0.48\textwidth}
         \centering
        \begin{tikzpicture}
\begin{axis}[
    xlabel={$u$},
    ylabel={$f(x,u)$},
    legend pos=north west,
    grid=major,
    domain=-1.5:1.5,
    samples=100,
    axis lines = middle,
    legend style={draw=none},
    xtick=\empty, 
    ytick=\empty 
]
\addplot[black, thick, solid] {0.25*(x^2-1)^2 + 0.1*x};
\addlegendentry{$x=-0.1$}

\addplot[black, thick, dashed] {0.25*(x^2-1)^2};
\addlegendentry{$x=0$}

\addplot[black, thick, dotted] {0.25*(x^2-1)^2 - 0.1*x};
\addlegendentry{$x=0.1$}

\end{axis}
\end{tikzpicture}
\caption{Notice how as $x$ passes through 0, the global minimizer of $f(x,\cdot)$ jumps between two values.} \label{fig:multi-well}
     \end{subfigure}
     \hfill
     \begin{subfigure}[b]{.48\textwidth}
         \centering
         \begin{tikzpicture}
\begin{axis}[
    axis lines = middle,
    xlabel = \(x\),
    ylabel = {\(u\)},
    enlargelimits,
    ymin=-1.5, ymax=1.5,
    xmin=-1, xmax=1,
    samples=100
]

\addplot[black, thick, domain=-1.3247:-1]({x^3-x}, x);

\addplot[black, thick, domain=1:1.3247] ({x^3-x}, x);
\end{axis}
\end{tikzpicture}
\caption{As a consequence, the minimizer of $\I(u)$ is discontinuous.}\label{fig:multi-well-minimizer} 
\end{subfigure}
\end{figure}
\end{exmp}

In the previous example we could immediately verify that $u^*$ is a minimizer, by directly comparing its energy to that of any other function. However, if we did not know the form of $u^*$ we would need to utilize the direct method to prove that a minimizer exists. For the sake of illustration, we will discuss this first in the context of the functional $\I$. As evidenced by the previous example, we need to minimize over a function space that permits discontinuities; we select $L^\infty([-1,1];\R)$ for simplicity.

When minimizing \eqref{dub_well} over the space of bounded functions, we notice that continuity of the energy with respect to the strong topology (i.e. the topology induced by the $L^\infty$ norm) is nearly immediate, because\[
  |\I(u_1)-\I(u_2)|\leq C\sup_{x\in [-1,1]}|u_1(x)-u_2(x)|.
\]However, the bounded sequences in $L^\infty$ are far from being compact: take for example $\mathrm{sign}(\sin(nx))$ which has no convergent subsequence in $L^\infty$. The standard approach is to weaken the notion of convergence on $L^\infty$ to convergence in duality with $L^1$, i.e. weak-* convergence, which we denote by $\rightharpoonup^*$. More explicitly, we say that $u_n \rightharpoonup_p^* u \in L^p$ if for every $v \in L^{p^*}$ we have that
\[
  \int u_n(x) v(x)\,dx \to \int u(x) v(x) \,dx, \qquad \frac{1}{p} + \frac{1}{p^*} = 1.
\]
We can directly check that $\mathrm{sign}(\sin(nx))\rightharpoonup_\infty^* 0$, and indeed we can show that any bounded sequence in $L^\infty$ is weak-* compact. However, the following example shows that upon moving to this topology the functional $\I$ is no longer lower semi-continuous.

\begin{exmp}\label{loss_of_weak_LSC}
Define the sequence $u_n(x)=\mathrm{sign}(\sin(n\pi x))$. As we have said, the sequence has no (strongly) convergent subsequence in $L^\infty([-1,1];\R)$, but $u_n\rightharpoonup_\infty^* 0$. However, it can be checked directly that $\I(u_n)=(-1)^n/n \to 0$, and that $\I(0)=1/2$. Therefore $\I$ is not weakly-* lower semi-continuous. Another way of interpreting this example is to notice that $f(x,u_n(x))$ does not converge in the weak-* topology to $f(x,0)$.\end{exmp}

We see in the previous example that the continuity of $\I$ with respect to $L^\infty$ does not imply that it is weak-* lower semi-continuous: the following classical result links this phenomenon with convexity for general integral functionals. For a reference, see Theorems 6.54 $\&$ 6.56 in \cite{book}.
\begin{prop}\label{weak-LSC}
Let $f:\R^d\times \R^m\to \R$ be a continuous function that is bounded below. For $1\leq p\leq \infty$ define $\I : L^p(\R^d;\R^m)\to \R$ by \[
\I(u)=\int f(x,u(x))dx,
\]then $I$ is weakly lower semi-continuous (weak-* if $p=\infty$) if and only if $u\mapsto f(x,u)$ is convex.
\end{prop}

For a simple integral energy of the form $\I$, it is possible to show existence of minimizers using direct pointwise optimization arguments. However, even in that case the limits of approximate minimizers may fail to be functions, highlighting potential issues for computational algorithms. Furthermore, the introduction of other terms in the functional, such as marginal constraints in optimal transportation, can make the existence of minimizers a very challenging problem. In our second-order case, the form of the energy is different, and we are aware of no direct construction of minimizers. In particular, we notice that the dimension reduction problem can be restated as \[
\min_{T:\R^d\to\R^m}\int J_T(x,T(x))\mu(dx),\quad \text{with}\quad J_T(x,y):= \int c(x,x',y,T(x'))\mu(dx').
\]As stated in the introduction, the cost function $c(x,x',y,y')$ is often not convex in practice, and in many cases we will not generally have that $y\mapsto J_T(x,y)$ is convex. Thus, by Proposition \ref{weak-LSC}, we suspect that the dimension reduction problem \eqref{eqn:MDS-general} will not be weakly lower semi-continuous. The following result demonstrates that this indeed is the case. 

\begin{prop}\label{DR_lacks_weak_LSC}
  Consider the dimension reduction problem \eqref{eqn:MDS-general} in the case where $c(x,x',y,y') = \tilde c(x,x',|y-y'|^2)$ for some $C^1$ function $\tilde c$ which is symmetric in $x,x'$. Assume that $\mu$ has a continuous density on an open and bounded set, and suppose that for all $x \neq x'$ we have that $\frac{d}{dt} \tilde c(x,x',t) |_{t=0} < 0$. Then the dimension reduction problem is not weakly lower semi-continuous.
\end{prop}
\begin{proof}Let us choose \[
T_n(x)=v \bigg(\prod_{i=1}^d \mathrm{sign}(\sin(n\pi x_i))\bigg)
\]for some $v\in\R^m$ which will later be specified. First note that clearly $T_n\rightharpoonup 0$. Furthermore, by denoting the sets\[
E_n=\left\{x:\prod_{i=1}^d \mathrm{sign}(\sin(n\pi x_i))=1\right\},
\qquad O_n=\left\{x:\prod_{i=1}^d \mathrm{sign}(\sin(n\pi x_i))=-1\right\},
\]the cost of $T_n$ will be computed as\begin{align*}
\J(T_n)&=\iint_{E_{n}\times E_{n}}c(x,x',0)\mu(dx)\mu(dx')\\
&+\iint_{O_{n}\times O_{n}}\tilde c(x,x',0)\mu(dx)\mu(dx')\\
&+2\iint_{E_{n}\times O_{n}}\tilde c(x,x',2|v|)\mu(dx)\mu(dx')\\
&=\J(0)+2\iint_{E_{n}\times O_{n}}[\tilde c(x,x',2|v|)-\tilde c(x,x',0)]\mu(dx)\mu(dx').
\end{align*}
Notice that we have, by the Riemann-Lebesgue Lemma, \[
2\iint_{E_{n}\times O_{n}}[\tilde c(x,x',2|v|)-\tilde c(x,x',0)]\mu(dx)\mu(dx')\xrightarrow{n\to\infty}\frac{1}{2}\iint[\tilde c(x,x',2|v|)-\tilde c(x,x',0)]\mu(dx)\mu(dx').
\]
where we have used the fact that
\[
  \mathds{1}_{E_n}(x) \mathds{1}_{O_n}(x') = \frac{(1+\prod_{i=1}^d \mathrm{sign}(\sin(n\pi x_i)))(1+\prod_{i=1}^d \mathrm{sign}(\sin(n\pi x_i')))}{4}
\]
along with the continuity of $c$ and the density $\mu$. Thus, given $\varepsilon>0$ for sufficiently large $n$, we have that
\begin{align*}
  \J(T_n)-\J(0)&<\frac{1}{2}\iint[\tilde c(x,x',2|v|)-\tilde c(x,x',0)]\mu(dx)\mu(dx')+\varepsilon \\
  &\leq \frac{1}{2} \iint -\phi(x,x') |v| + o(|v|) \mu(dx) \mu(dx') + \varepsilon,
\end{align*}
where $\phi(x,x') \geq 0$ with equality only possibly when $x=x'$ by our assumption upon the derivative of $\tilde c$. Making $v$ sufficiently small so that we can neglect the $o(|v|)$ term, and taking $\e \to 0$ then implies that $\liminf_n \J(T_n) < \J(0)$, proving the result.
\end{proof}

The previous proposition demonstrates that the dimension reduction energy $\J$ is not weakly lower semi-continuous: this implies that information about minimization is lost in limit obtained with that topology. The standard approach to handling this situation is to instead permit limits that are multi-valued: meaning that one $x$ is mapped probabilistically to multiple $y$ values. For example, in the proof of the previous proposition we may write
\[
  \pi_n(dx\: dy) = \mu(dx) ( \mathds{1}_{E_n}(x) \delta_{v}(dy) + \mathds{1}_{O_n}(x) \delta_{-v}(dy)),
\]
and then compute
\begin{displaymath}
  \J(T_n) =  \iint c(x,x',y,y') \pi_n(dx\: dy) \pi_n(dx'dy').
\end{displaymath}
Using the computation with the Riemann-Lebesgue lemma in the proof of the previous proposition, it is straightforward to show that $\pi_n$ converges (in the sense of weak convergence of measures) to $\pi(dx\: dy) = \mu(dx) (1/2 \delta_{v}(dy) + 1/2 \delta_{-v}(dy))$. Hence we have that
\begin{displaymath}
  \J(T_n) \to \iint c(x,x',y,y') \pi(dx\: dy) \pi(dx'dy').
\end{displaymath}
Slightly abusing notation, we can then define a \emph{relaxed energy}
\[
  \J(\pi) := \iint c(x,x',y,y') \pi(dx\: dy) \pi(dx'dy').
\]
Here $\pi$ represents a probabilistic coupling between $x$'s and $y$'s which generalizes a deterministic coupling (or function) mapping each $x$ to a single $y$. In the context of optimal transportation, the coupling $\pi$ is sometimes called a \emph{transportation plan}, whereas a deterministic coupling in that context is called a \emph{transportation map}. In the continuum mechanics literature such a probabilistic relaxation is called a \emph{Young measure}. In many contexts the existence of minimizers of the relaxed energy is more straightforward to prove using compactness and continuity arguments: we carry out these standard arguments in the next section.

\subsection{Existence of relaxed solutions}\label{sec:ex-relaxed}

In light of the discussion in the previous section, we turn our attention to the problem of existence of minimizers to the relaxed problem \eqref{DR}. We begin by giving some definitions.
%
Given $\mu\in\Pp(\R^d)$ and family $\{\nu(\cdot|x)\}_{x\in\R^d}\subset \Pp(\R^m)$ for which $x\mapsto \nu(Q|x)$ is a measurable function for all $Q\in\B(\R^m)$, there exists a unique (in measure) probability distribution $\pi\in\Pp(\R^d\times \R^m)$ such that for all $P\in\B(\R^d)$ and $Q\in \B(\R^m)$,\begin{equation}\label{up_disint}
  \pi(P\times Q)=\int_P \nu(Q|x)\mu(dx). \end{equation}
Let the space of all joint probability measures which can be written in the form above be called $\Pi(\mu)$, more precisely $\Pi(\mu):=\{\pi\in\Pp(\R^d\times \R^m)\,|\,\mathrm{proj}_{\R^d}\sharp \pi = \mu\}$ which are all the probability measures on $\R^d\times \R^m$ with $\X$-marginal $\mu$. In analogy to optimal transportation, we call $\Pi(\mu)$ the set of \emph{embedding plans} for $\mu$. 

As soon as $c$ is itself lower semi-continuous, the function $\pi\mapsto \int\hspace{-0.2cm}\int c\,d\pi d\pi$ is automatically lower semi continuous with respect to weak convergence of probability measures, by Portmanteau's theorem. We recall that a sequence of probability measures $\pi_n \in \Pp(\R^d \times \R^m)$ is said to converge weakly to $\pi$ if for every bounded, continuous function $f$ we have that $\int f d\pi_n \to \int f d\pi$. In order to recover sequential compactness for sequences of probability measures $\pi_n\in\Pi(\mu)$, we must introduce the notion of \emph{tightness of measure} and its application on the subspace $\Pi(\mu)$.
\begin{defn}[Tightness of Embedding Plans]\label{Embedding Plans}
A sequence of probability distributions $\{\pi_n\}_{n=1}^\infty\subset\Pp(\R^d\times\R^m)$ is said to be tight if for every $\varepsilon>0$, there exists a compact set $K_\varepsilon\subset\R^d\times \R^m$ for which $
\sup_{n}\pi_n(K_\varepsilon^c)<\varepsilon.
$

In the case that $\pi_n\in\Pi(\mu)$, we can find a compact set $K_d$ in $\R^d$ so that $\mu(K_d) > 1-\frac{\e}{2}$. In turn if we can find a compact set $K_m$ so that $\pi_n(\R^d \times K_m) >1-\frac{\e}{2}$ we can use $K = K_d \times K_m$ and obtain the estimate $\pi_n(K^c) < \e$: this implies that when $\pi_n \in \Pi(\mu)$ we only need to verify tightness in the marginal over the last $m$ coordinates. In symbols, we write this as
\[
\{\pi_n\}_{n=1}^\infty \text{ is tight in }\Pi(\mu)\iff \nu_n:=\mathrm{proj}_\Y\sharp\pi_n,\,\{\nu_n\}_{n=1}^\infty\text{ is tight in }\Pp(\R^m).
\]
\end{defn}
Here we, in a slight abuse of notation, are letting $\nu(Q) = \int_{\R^d \times Q} \,d\pi(x,y)$: meaning that if we suppress the $x$-dependence in $\nu(dy|x)$ then we are indicating the marginal distribution in $y$. 

By Prokhorov's theorem, tightness of a sequence of probability measures implies weak compactness.  Thus the problem of existence of minimizers to the relaxed problem reduces to establishing tightness of sequences of embedding plans with bounded energy $\J$. 
\subsection*{Assumptions}
\label{subsection:assumptions}
We are now ready to list our assumptions. As stated before, we will consider the following two types of costs:
\begin{alignat*}{2}
	 &(\mathbf{IP})\quad && c(x,x',y,y')=\tilde c\big(x,x',\langle y,y'\rangle\big)\\
	 &(\mathbf{N}^2)\quad && c(x,x',y,y')=\tilde c\big(x,x',|y-y'|^2 \big)
	 \end{alignat*}
	 where we make the following assumptions on the function $\tilde c:\X\times \X\times \R\to\R$.
\begin{alignat*}{2}
	 &(\mathbf{A1}) \quad && \textit{For every compact set } K\subset \R^d, \textit{ there is an unbounded increasing function}\\
	 & &&f_K:\R^+\to\R\textit{ such that }x,x'\in K\implies \tilde c(x,x',t)\geq f_K(t)\geq 0\\
	 &(\mathbf{A2})\quad && \textit{For } \mu\otimes\mu\textit{-a.e. } (x,x'),\,\,
	 \tilde c(x,x',t)=\tilde c(x',x,t),\,\,\forall t\in \mathbb{R}\\
	 &(\mathbf{A3})\quad &&\textit{For } \mu\textit{-a.e. }x,\,\,t\mapsto \tilde c(x,x,t)\,\, \textit{has a unique minimizer at}\,\,t=0
\end{alignat*}
Furthermore, we make the following assumptions on the growth of the derivatives which are most clearly stated in terms of $c$ rather than $\tilde c$:
\begin{alignat*}{2}
  &(\mathbf{A4})\quad&& c \textit{ is a }C^2\textit{ function in all its variables with derivative values satisfying } |D^2\, c|\leq C(1+ c)\\
	 &(\mathbf{A5})\quad&& \textit{For any }M > 0 \textit{ there exists a }\delta>0 \textit{ and non-negative, strictly increasing continuous functions }\\
	 &\quad &&\psi_1,\psi_2: \R^+ \to \R^+\textit{ satisfying }\psi_1(0) = \psi_2(0) = 0\textit{ so that for any }|x| < M,\, |x-x'|<\delta\\
	 &\quad &&\textit{ and for any }y,y' \textit{ we have }
    D_{y y'}^2 c(x,x',y,y') < -\psi_1(|y-y'|^2) +  \psi_2(|x-x'|^2),
\textit{ where here }\\
&\quad&&\textit{the inequality is meant in the sense of positive definite matrices. }
	 \end{alignat*}
Assumption ($\mathbf{A1}$) ensures that $c$ is nonnegative as well as provides coercivity. Assumptions $(\mathbf{A3})$-$(\mathbf{A5})$ are listed here for completeness, but are not used in the proofs of relaxed existence. The growth condition $\mathbf{(A4)}$ will allow us to integrate derivatives in a meaningful way. This assumption on growth conditions of derivatives of $c$ naturally holds for polynomial costs. Assumptions $\mathbf{(A3)}\&\mathbf{(A5)}$ are intended for costs of the form $\mathbf{(N}^2)$ and are also widely applicable. 
One last assumption we list separately as it is stronger than necessary but encompasses many relevant costs is \begin{alignat*}{2}
	 	 &(\mathbf{A0})\quad && \textit{For } \mu\otimes\mu\textit{-a.e. } (x,x'),\,\,t\mapsto \tilde c(x,x',t)\,\,\textit{is strictly convex}
	 \end{alignat*}
As a final note, we mention that unless otherwise specified we will drop the tilde on the cost in the above assumptions. For example, we will write $c(x,x',|y-y'|^2)$ rather than $\tilde c(x,x',|y-y'|^2)$. Provided below is a table of several cost functions which can fit into our framework.
\begin{table}[h]
\centering 
\begin{tabular}{|c|c|}
\hline
\rule{0pt}{3ex}Method & $c\big(x,x',y,y'\big)$\\[0.125cm]
\hline
\rule{0pt}{3ex}PCA & $\big(\langle x,x'\rangle-\langle y,y'\rangle\big)^2$ \\[0.125cm]
\hline
\rule{0pt}{3ex}Kernel PCA & $\big(\kappa (x,x')-\langle y,y'\rangle\big)^2$ \\[0.125cm]
\hline
\rule{0pt}{3ex}q-MDS$ $& $\big(|x-x'|^2-|y-y'|^2\big)^2$\\[0.125cm]
\hline
\rule{0pt}{4ex}q-Sammon$ $&$\tfrac{\big(|x-x'|^2-|y-y'|^2\big)^2}{|x-x'|^2}$\\[0.125cm]
\hline
\rule{0pt}{4.5ex}Elastic Embeddings & $\begin{array}{c} |y-y'|^2e^{-\tfrac{|x-x'|^2}{2\sigma^2}} \\ +\beta |x-x'|^2 e^{-|y-y'|^2} \end{array}$ \\[0.125cm]
\hline
\end{tabular}
\caption{A list of several costs which fit into our framework and satisfy Assumptions ($\mathbf{A1}$)-($\mathbf{A5}$); notice that the Elastic Embedding cost is one which does not satisfy assumption $(\mathbf{A0})$ yet does satisfy $(\mathbf{A3}$). The ``q'' refers to quartic variants of standard costs used in dimension reduction.}
\label{applicable_costs}
\end{table}

We now explicitly derive an upper bound which quantifies tightness under the assumption $(\bf{A1})$. We begin with the inner product case.

\begin{theorem}[Inner Product Costs]
\label{existence}
    Assume $\mathbf{(IP)}$ and $(\mathbf{A1})\,\&\,(\mathbf{A2})$ and that $c$ is lower semi-continuous. Let $\mu\in\Pp(\R^d)$ and suppose that $\inf_{\Pi(\mu)} \J<+\infty$, where $\J$ is given by \eqref{DR}. Then there exists $\pi \in \Pi(\mu)$ such that $\J(\pi)=\inf_{\Pi(\mu)} \J.$
\end{theorem}
\begin{proof} We consider a sequence $\pi_n$ so that $\J(\pi_n) \to \inf_{\Pi(\mu)} \J$ and $$\J(\pi_n) \leq 2 \inf_{\Pi(\mu)} \J.$$ Notice that if $\langle y,y'\rangle=0$ for $\nu\otimes \nu$-a.e. $(y,y')$, it must be that the support of $\nu$ is concentrated on the singleton $\{0\}$, which would trivially give tightness of $\pi_n$; thus without loss of generality we may assume that elements of the minimizing sequence have nontrivial support in $y$. 

  We now claim that the sequence $\pi_n$ must be tight: the argument will essentially show that mass far from the origin must be small in order for the previously displayed inequality to hold. As described in the definition of tightness, it suffices to show that $\nu_n$ is tight. To begin, we let $\e>0$ and partition $\R^m$ into a finite number of disjoint cones $C_1,...,C_\ell$ wherein the angle between any two points is at most $\pi/6$ and denote $C_{i,r}=C_i\cap B_r^c(0)$ for $i=1,...,\ell$. Let $K_\e \subset \R^d$ be a compact set such that $\mu(K_\e^c)<\frac{\e}{2}$. By the non-negativity of $c$ which follows from $(\mathbf{A1})$, we have
\[
\J(\pi_n)\geq \sum_{i=1}^\ell \iint_{(K_\e\times C_{i,r})^2} c\,d\pi_n d\pi_n,
\]
which, by assumption $(\mathbf{A1})$, yields
\[
\J(\pi_n) \geq \sum_{i=1}^\ell \iint_{(K_\e \times C_{i,r})^2} f_{K_\e}\circ |\langle \cdot,\cdot \rangle|\,d\pi_n d\pi_n. 
 \] Finally, by the construction of our cones, we have that $y,y'\in C_i \implies |\langle y,y'\rangle| \geq |y|\,|y'|/2$, and hence
\[
\J(\pi_n)\geq f_{K_\e}(r^2/2) \sum_{i=1}^\ell \big(\pi_n(K_\e \times C_{i,r})\big)^2\geq f_{K_\e}(r^2/2) \frac{\big(\pi_n(K_\e \times B_r^c(0))\big)^2}{\ell}.
\]
The second inequality follows by Jensen's inequality and by virtue of $C_1,...,C_\ell$ forming a partition. The above considerations hence imply for every element of the minimizing sequence, one has\[
\pi_n(\R^d\times B_r^c(0)) = \nu_n(B_r^c(0)) \leq \sqrt{\frac{2 \ell \inf_{\Pi(\mu)\J }}{f_{K_\e}(r^2/2)}} + \frac{\e}{2}.
\] 
By then making $r$ sufficiently large we can make the right hand side smaller than $\e$, which shows that the $\nu_n$, and subsequently the $\pi_n$, are tight. Prokhorov's Theorem gives a subsequence with a weak limit $\pi$, and $\pi$ is a relaxed minimizer by the weak lower semi-continuity of $\J$, as argued above.
\end{proof}

The same argument, with only slight modifications to the geometry, provides the same result for the norm-based costs.

\begin{theorem}[Normed Costs]
\label{existence_NC}
Assume $(\mathbf{N}^2)$ and $(\bf{A1})\,\&\,(\bf{A2})$ and that $c$ is lower semi-continuous. Let $\mu\in\Pp(\R^d)$ and suppose that $\inf_{\Pi(\mu)} \J<+\infty$, where $\J$ is given by \eqref{DR}. Then there exists $\pi \in \Pi(\mu)$ such that $\J(\pi)=\inf_{\Pi(\mu)} \J.$
\end{theorem}

\begin{proof}
The main difference in the proof is that one should replace cones, which have aligned inner products, with pairs of halfspaces which are well-separated, and hence have lower bounds on pairwise distances. 

Specifically, let $\{\pi_n\}_{n=1}^\infty$ satisfy  $\J(\pi_n) \to \inf_{\Pi(\mu)} \J$ and $\J(\pi_n) \leq 2 \inf_{\Pi(\mu)} \J$. Since the cost is translation invariant in $y$, without loss of generality, we may assume that each element in this sequence has the property that for any $k \in 1 \dots m$ we have $\pi_n(\R^d\times H_k^+)=\pi_n(\R^d\times H_k^-)=1/2$ where $H_k^+:=\{y\in\R^m : y_k>0\}$ and $H_k^-:=\{y\in\R^m : y_k\leq 0\}$. We also write $H_{k,r}^+=\{y\in\R^m : y_k>r\}$. As before, take $K_\e \subset\R^d$ to be a compact set for which $\mu(K_\e^c)<\frac{\e}{4m}$, and let $\e < 1/2$. By the non-negativity of $c$, one has for any $k \in 1 \dots m$
\[
\J(\pi_n) \geq \iint_{(K_\e\times H_{k,r}^+ )\times (K_\e \times H_{k}^-)} c\,d\pi_n d\pi_n.
\]
By again using the bound $(\mathbf{A1})$, the monotonicity and unboundedness of $f_{K_\e}$, and the fact that $(y,y')\in H_{k,r}^+\times H_{k}^-\implies |y-y'|^2>r^2$, then gives, for $r$ sufficiently large,
\[
\frac{\J(\pi_n)}{f_{K_\e}(r^2)} \geq \pi_n(K_\e \times H_{k,r}^+) \pi_n(K_\e \times H_k^-) \geq \left(\nu_n(H_{k,r}^+) - \frac{\e}{2m}\right)\left(\frac{1}{2} - \frac{\e}{2m}\right)
\]
and in turn, rearranging, summing over $k$, and using the fact that $\e < 1/2$, we obtain
\[
\nu_n(\cup_{k=1}^m H_{k,r}^+) \leq m \frac{8 \inf_{\Pi(\mu)} \J}{f_{K_\e}(r^2)} + \frac{\e}{4}.
\]
By repeating the argument for the halfspaces where $y_k < -r$, we then obtain
\[
\nu_n( \{ |y|_\infty >r) \leq  \frac{16 m  \inf_{\Pi(\mu)} \J}{f_{K_\e}(r^2)} + \frac{\e}{2},
\]
and by taking $r$ sufficiently large we can then bound $\nu_n( \{ |y|_\infty >r) \leq \e$. This proves tightness of the $\nu_n$, which in turn proves, up to a subsequence, existence of a weak limit $\pi$ which must be a minimizer.
\end{proof}

\section{The Marginal Problem}\label{sec:marginal}
As discussed in the introduction, many of the standard tools for existence of transportation maps in optimal transportation fail in the present context due to a lack of convexity in $\pi$ of the relaxed problem. In particular, the effects of replacing an embedding plan $\pi$ with $\pi+\dpi$ (such that $\pi+\dpi\in \Pi(\mu)$) are realized as first \emph{and} second-order terms in $\dpi$. More precisely, if $\dpi$ is a signed measure on $\X\times \Y$ such that for all $d$-dimensional Borel sets $A$,  $\dpi(A\times \R^m)=0$, one has \[
 \mu(A)=\pi(A\times \R^m)=[\pi+\dpi](A\times \R^m),
 \]so that adding $\dpi$ leaves the $\X$-marginal invariant. With this notation along with the symmetry assumption in $\mathbf{(A2)}$, one can succinctly express the change in energy due to the perturbation $\dpi$:\begin{equation}\label{quadratic_perturbation}\J(\pi+\dpi)-\J(\pi)=2\underbrace{\iint c\,d\pi d\dpi}_{=:\J(\dpi |\pi)}+\underbrace{\iint c\,d \dpi d \dpi}_{=:\J(\dpi)},.
\end{equation}
Here, $\dpi\mapsto \J(\dpi |\pi)$ encapsulates the linear contribution while $\dpi \mapsto \J(\dpi)$ represents the quadratic contribution. Further developing this notation, we remark that the $\J(\dpi |\pi)$ encodes the fact that the first-order affect should be thought of as a linear programming problem over $(x,y)\mapsto \int c(x,x',y,y')\pi(dx'dy')$ for a fixed embedding plan $\pi$. Denoting this map as $J_\pi(y|x)$, we see that the first-order problem can be formally stated: for any fixed $\tilde \pi\in\Pp(\R^d\times \R^m)$, find $\pi$ such that\[
\pi\in\arg\min_{\Pi(\mu)} \int J_{\tilde \pi}(y|x) \pi(dx\:dy)
\]As we are free to vary the $\Y$-marginal of $\pi$, the above formulation strongly suggests that if $\pi(dx\:dy)=\nu(dy|x)\mu(dx)$ is optimal, then the support of $\nu(\cdot|x)$ is concentrated on the minimizers of $J_\pi(\cdot|x)$. This turns out to indeed be the case, but before validating the claim, we give a definition to streamline the proceeding discussion.
\begin{defn}
\label{defn:marginal_problem}
Given a continuous cost $c$ of type $(\mathbf{IP})$ or $(\mathbf{N}^2)$ which satisfies assumptions $(\mathbf{A1}) \& \mathbf{(A2)}$ and a embedding plan $\pi\in\Pi(\mu)$, we define the marginal problem of $\J(\pi)$ by the function \begin{equation}
\label{MP}
    J_\pi(y|x):=\int  c(x,x',y,y')\pi(dx'dy').
\end{equation}
Furthermore, for the set valued map $\lambda: x\mapsto \arg\min J_\pi(\cdot|x)$, we call the set of all pairs $(x,\lambda(x))$ the minimal graph of $J_\pi$ and denote it with the symbol $\Lambda_\pi$.
\end{defn}
\noindent Notice that the chosen convention is that calligraphic letters are reserved to functional problems while standard capital letters denote functions on finite dimensional spaces. We also remark that when $c$ is continuous, With this definition in place, we now present the following theorem.
\begin{theorem}[Marginal Minimization]
\label{marginal_problem}
Suppose that $c$ is a continuous cost of type $(\mathbf{IP})$ or $(\mathbf{N}^2)$ and satisfies assumptions $(\mathbf{A1}) \& \mathbf{(A2)}$. If $\pi\in\Pi(\mu)$ is a minimizer of \eqref{DR}, then the support of $\pi$ is concentrated on the minimal graph of $J_{\pi}$. In other words, $\pi$ must satisfy the implicit relation
 \begin{equation}\pi(\Lambda_\pi)=1.\end{equation} \end{theorem}

From a high level, the theorem tells us that the variational problem \eqref{DR} may be transformed into a finite dimensional one; that of minimizing $J_\pi(\cdot |x)$ for every given $x$ (which implicitly depends on $\pi$). This is analogous to the situation in optimal control wherein a value function is found by solving a PDE which implicitly depends on the control $u$. Once this value function is found, one may pointwise minimize a (finite dimensional) Hamiltonian to find the optimal control. 

Continuing the analogy with control, notice that in the absence of a convexity assumption on $c$, smoothly varying $\pi$ is likely prone to get `stuck' in local minima. To this end, the proof of the theorem uses localized perturbations in $\X$ which transport probability mass in $\Y$ across potentially large distances. These perturbations are analogous to needle variations used in the proof of the Pontryagin Maximum Principle.

We now illustrate the proof idea in the discrete case. To this end, suppose $\mu=(1/n)\sum_i \delta_{x_i}$ and $\pi=(1/n)\sum_{ij}\pi_{ij}\delta_{(x_i,y_j)}$ where $y_1,y_2,...,y_n \in\R^m$ constitute the optimal solution to \eqref{DR}; each $\pi_{ij}$ tells what proportion of the $1/n$ mass at point $x_i$ will go to location $y_j$. Suppose that $y_j\not\in \lambda(x_i)$ for some pair $(x_i,y_j)$ with $\pi_{ij}>0$. Define a perturbation $\dpi$ which sends $y_j$ to $\tilde y_j\in\lambda(x_i)$, that is \[
\dpi= \frac{\pi_{ij}}{n}\Big(\delta_{(x_i,\tilde y_j)}-\delta_{(x_i,y_j)}\Big)
\]
and let $\tpi = \pi+\dpi.$ Computing first the effect on the linear term, $\J(\dpi|\pi)$ we have \begin{align*}
\J(\dpi|\pi)&=\frac{1}{n}\Big[ J_\pi(\tilde y_j |x_i)-J_\pi(y_j|x_i)\Big]<0,
\end{align*}by marginal minimality of $\tilde y_j$. Further, we have\[\dpi\otimes \dpi=(1/n^2)\Big(\delta_{(x_i,\tilde y_j)}\otimes \delta_{(x_i,\tilde y_j)}-\delta_{(x_i,\tilde y_j)}\otimes \delta_{(x_i, y_j)}-\delta_{(x_i, y_j)}\otimes \delta_{(x_i,\tilde y_j)}+\delta_{(x_i, y_j)}\otimes \delta_{(x_i, y_j)}\Big),\]and hence $\J(\dpi)=1/n^2\big[c(x_i,x_i,y_j, y_j)-2c(x_i,x_i,y_j,\tilde y_j)+c(x_i,x_i,\tilde y_j,\tilde y_j)\big]$ which is clearly dominated by the linear term when $n$ is large enough. Thus by Equation \eqref{quadratic_perturbation}, $\J(\pi+\dpi)<\J(\pi)$ and we obtain a contradiction to the optimality of $\pi$. Extending this idea to the continuum case only requires a direct, measure-theoretic argument.

 \begin{proof}
   Let $\pi$ be an optimal solution of \eqref{DR} and suppose for sake of contradiction that  $\pi(\Lambda_\pi^c)>0$. By defining \[A_{k,r}=\left\{(x,y):k^{-1} <J_\pi(y|x)- \min J_\pi(\cdot|x)\right\}\cap \{(x,y):|x|,|y|<r\},\]it follows that $\Lambda_\pi^c=\bigcup_{k,r=1}^{\infty} A_{k,r}$ and consequentially, $\pi(A_{k,r})>0$ for some $(k,r)\in\mathbb{N}^2$. Define the measure $\pi_{k,r}=\frac{\pi_{|A_{k,r}}}{\pi(A_{k,r})}$ and take $\tilde \lambda$ as a measurable selection of $\lambda$. This selection exists by the continuity of the marginal problem, $J_\pi$, which follows by the continuity of $c$\footnote{The existence of a minimizing measurable selection of $J_\pi$ follows from a theorem of Rockafeller (see 14.37 in \cite{rockafellar2009variational}) as soon as $J_\pi$ is a Carathéodory function.}. Choose $\varepsilon < \min\{2\pi(A_{k,r}),(k \|c\|_{L^\infty(A_{k,r}\times A_{k,r})})^{-1}\}$  to construct the perturbation \[
 \dpi = \frac{\varepsilon}{2} \bigg (\frac{\nu(A_{k,r}|x)}{\pi(A_{k,r})}\cdot\delta_{\tilde \lambda(x)}\otimes \mu -\pi_{k,r}\bigg).
 \]
where we have used the representation $\pi(A_{k,r})=\int \nu(A_{k,r}|x)\mu(dx)$. By the first restriction on $\varepsilon$, it follows that $\pi+\dpi$ is a positive measure. Furthermore we can see that this perturbation does not affect the input marginal, that is $\dpi(P\times \R^m)=0$ for all $P\in \mathcal{B}(\R^d)$.

Tracking the effects of this perturbation, the linear term becomes:
\begin{align*}
\J(\dpi|\pi)&=\frac{\varepsilon}{2\pi(A_{k,r})}\int J_\pi(y|x)\delta_{\tilde \lambda(x)}(dy) \nu(A_{k,r}|x)\mu(dx)-\frac{\varepsilon}{2}\int J_\pi(y|x)\pi_{k,r}(dx\:dy)\\
&<\frac{\varepsilon}{2\pi(A_{k,r})}\int \min J_\pi(\cdot |x)\nu(A_{k,r}|x)\mu(dx)-\frac{\varepsilon}{2}\int \big(\min J_\pi(\cdot |x) +k^{-1} \big)\frac{\nu(A_{k,r}|x)\mu(dx)}{\pi(A_{k,r})}\\
&=-\frac{\varepsilon}{2k}
\end{align*}where on the second to last line we make use of the lack of dependence on $y$ in the latter integrand. As $c$ is nonnegative, we have the following estimate for the quadratic term:
\begin{align*}
\J(\dpi)&\leq \frac{\varepsilon^2}{4}\iint c(x,x',\lambda(x),\lambda(x'))\nu(A_{k,r}|x)\mu(dx)\nu(A_{k,r}|x')\mu(dx')\\
&+\frac{\varepsilon^2}{4}\iint c(x,x',y,y')\pi_{k,r}(dx\:dy)\pi_{k,r}(dx'dy')\\
&\leq  \varepsilon^2 \cdot \|c\|_{L^\infty(A_{k,r}\times A_{k,r})}.
\end{align*}Putting the estimates together with \eqref{quadratic_perturbation}, one has\[
  \J(\pi+\dpi)-\J(\pi)<-\frac{\varepsilon}{k}+ \varepsilon^2 \cdot \|c\|_{L^\infty(A_{k,r}\times A_{k,r})}
\]which is negative by our choice of $\varepsilon$. This is a contradiction to optimality.
 \end{proof}

\begin{remark}\label{power_method}
  In the proof presented above we notice that transporting $\e$-mass to (global) marginal minimizers incurs a gain on the embedding cost regardless of whether or not $\pi$ is optimal. This is quite different in philosophy from the standard computational approaches which conduct particle-wise gradient descent in $\Y$. As evidenced by Example \ref{exmp:numerical}, particle-wise decent potentially gets caught in local minima of the marginal problem. These local minima can lead to highly oscillatory embeddings: in the language of this work this corresponds to probabilistic couplings.

  A different way of casting this observation is that if we are only allowed to perturb a coupling $\pi$ smoothly in $y$ then there may be local minimizers of $\J$ which are probabilistic in $\Y$. However we shall see in Section \ref{sec:Normed_costs} that probabilistic couplings are never optimal in our dimension reduction problems. This suggests the need for improved computational algorithms which are capable of executing perturbations which are not smooth in $\Y$.

\end{remark}

\subsection{Critical point equation}
In light of Theorem \ref{marginal_problem}, it is natural to consider the necessary conditions for optimality in $y$ of the marginal problem, and the constraints that they impose upon the optimal solution $\pi$. To begin, we consider assumptions under which the marginal problem, which depends implicitly upon the measure $\pi$, is differentiable.

\begin{lemma}
  Let the cost function $c$ be of type $(\mathbf{IP})$ or $(\mathbf{N}^2)$ and satisfy assumptions $(\mathbf{A1}), \mathbf{(A2)}, $ and $(\mathbf{A4})$. Let $\pi$ be a minimizer of \eqref{DR}. Then the funtion $J_\pi$ is $C^2$ in $x,y$.
  \label{lem:marginal-problem-smooth}
\end{lemma}

\begin{proof}
Formally differentiating we should have the formula
  \begin{displaymath}
    D^2 J_\pi(y|x) = \int D^2c(x,x',y,y') \pi(dx'dy').
  \end{displaymath}
  However, by $(\mathbf{A4})$, we can write
  \[
    \iint |D^2 c| d\pi d\pi \leq C(1 + \mathcal{J}(\pi)) < \infty.
  \]
  This in turn implies that $\int D^2 c(x,x',y,y') \pi(dx'dy')$ is integrable (with respect to $\pi$), in $x,y$. A dominated convergence argument, along with continuity of the derivatives, then gives that $J_\pi$ is $C^2$ in $x,y$.
\end{proof}

We notice that a necessary condition for optimality is that $\spt \pi$ must be concentrated on solutions to the nonlinear integral equation in $\X\times \Y$\begin{equation}\label{gen_CPE}
D_y J_\pi(y|x) = \int D_y c(x,x',y,y')\pi(dx'dy')=0.
\end{equation}As the goal is to establish that $y$ is deterministically given by $x$, if $y\mapsto D_y J_\pi(y|x)$ were injective then for every given $x$ the unique solution to $D_y J_\pi(y|x)=0$ would specify $y$. However, we do not expect this to be the case in general (see Example \ref{qMDS_MP}). This stands in contrast to the situation in optimal transportation wherein $D_x c(x,\cdot)$ is assumed injective (sometimes called the \emph{twist condition}) rendering the equation $D \psi(x)+D_x c(x,y)=0$ to be a prescription of $y$ given $x$. Notice how the presence of the Kantorivich potential $\psi$ somehow encodes the additional marginal constraint which is present in OT; in the absence of this constraint in the dimension reduction problem, it is unsurprising there is no analogous term in \eqref{gen_CPE}.

In special cases, it can happen that the marginal problem is strictly convex as a function of $y$. We begin with a simple example in the context of classical dimension reduction algorithms.
\begin{exmp}
\label{PCA_via_MP}
Let $c(x,x',y,y')=\big(\langle x,x'\rangle -\langle y,y'\rangle \big)^2$. Then the marginal problem takes the form\[
J_\pi(y|x)=x^T\left[\int x'{x'}^T\mu(dx')\right]x-2x^T\left[\int x'{y'}^T\pi(dx'dy')\right]y+y^T\left[\int y'{y'}^T\nu(dy')\right]y.
\]Clearly, $y\mapsto J_\pi(y|x)$ is convex and thus $D_y J_\pi(y|x)=0$ will determine $y$ given $x$. Writing the critical point equation, we see\begin{equation}\label{IP_CPE}
\left[\int y'{x'}^T\pi(dx'dy')\right]x=\left[\int y'{y'}^T\nu(dy')\right]y,
\end{equation}indicating that the optimal map is linear, meaning $y = Ax$. If we utilize the singular value decomposition $A = U \Sigma V^T$, we can rewrite the original optimization problem as
\[
  \J(\pi) = \J(A) = \iint (x^T x' - x^T V \Sigma^T \Sigma V^T x')^2 \mu(dx) \mu(dx') = \iint (x^TV^T(I - \Sigma^T \Sigma) Vx')^2 \mu(dx) \mu(dx').
\]
This is equivalent, for centered $\mu$, to principal component analysis. \end{exmp}

Building upon this example, we can give the following simple corollary to Theorem \ref{marginal_problem}.


\begin{corollary}
  Suppose that $c$ is of type $\mathbf{(IP)}$ and satisfies Assumptions $\mathbf{(A0)}$-$\mathbf{(A2)}$ and $\mathbf{(A4)}$. Then any optimal solution of \eqref{DR} is supported on the graph of a function, whose smoothness is controlled by the differentiability of $t\mapsto \tilde c(x,x',t)$.
  \label{corollary:IP_mins}
\end{corollary}

\begin{proof}
  We notice that the marginal problem \eqref{MP} is represented as an integral of $c$ integrated against a measure on $(x',y')$. Since $y\mapsto \langle y,y'\rangle$ is a linear function, and $t\mapsto c(x,x',t)$ is strictly convex, we have that the marginal problem is strictly convex on the smallest subspace supported by $\nu$. If that subspace is $\R^m$ then we immediately have that the marginal problem is strictly convex for almost every $x$, and hence has a unique minimizer for almost every $x$. In turn, the map $\lambda$ is actually a function, and $\pi$ is supported on the graph of that function. If the smallest subspace containing the support of $\nu$ is of dimension $k < m$, then by using rotational invariance we could instead consider the problem of embedding to $\R^k$, and the same argument as above gives that the solution must be induced by a map. Finally, the smoothness of the optimal map may be recovered by noting that solutions must solve the critical point equation $D_y J_\pi=0$ and using the implicit function theorem.

\end{proof}

This corollary resolves the necessity of optimal solutions to be mappings in many natural contexts, specifically costs which are convex in $\langle y, y' \rangle$. Such costs include classical multi-dimensional scaling and kernel principle component analysis. However, many of the standard costs used in dimension reduction are non-convex in $y$, and have marginal problems with more complicated structure in their minimizers. We return to our running example which demonstrates that the marginal problem can have multiple minimizers.

\begin{exmp}
\label{qMDS_MP}
In the case of $c(x,x',y,y')=\big(|x-x'|^2-|y-y'|^2\big)^2$, one has a rather explicit formula for the marginal problem:\begin{equation}\label{qMDS_CPE_cont}
J_\pi(y|x)=|y|^4-2y^T\psi_\pi(x) y-4\varphi_\pi(x)^T y+\zeta_\pi(x).
\end{equation}
The coefficients of this polynomial equation are implicitly defined by moments of the joint distribution, in particular: \begin{align*}
\psi_\pi(x)&=\mathrm{Id_m}|x|^2-\int\Big[ 2y'{y'}^T +(|y'|^2-|x'|^2)\mathrm{Id_m}\Big]\pi(dx'dy')\\
\varphi_\pi(x)&=2\underbrace{\left(\int y'{x'}^T\pi(dx'dy')\right)}_{=:\Phi_\pi}x+\int y'(|y'|^2-|x'|^2)\pi(dx'dy')\\
\zeta_\pi(x)&=|x|^4+4 x^T\left(\int x'{x'}^T\mu(dx')\right)x-2|x|^2\left(\int\big(|y'|^2-|x'|^2\big)\pi(dx'dy')\right)\\
&+4\left(\int \big(|y'|^2-|x'|^2\big){x'}^T\pi(dx'dy')\right)x+\int \big(|y'|^2-|x'|^2\big)^2\pi(dx'dy')
\end{align*}
where we have assumed the distribution in $\R^m$ has mean zero by using translation invariance. We notice that the matrices $\psi_\pi, \varphi_\pi,$ and $\zeta_\pi$, which are completely determined by moments of $\pi$, give a parametric representation for the marginal problem, just as $A$ did in the inner product case from the previous example. We believe that this parametric representation should be useful for many unsupervised learning tasks, as it will directly give properties such as statistical consistency and direct extrapolation. Furthermore, it should facilitate more efficient computational algorithms that work in parameter space: this is the subject of current work.

Let $\eta_1,...,\eta_m$ be an orthogonal basis for which \[
\sum_{j=1}^m \eta_j\eta_j^T=\int \Big[2yy^T+\big(|y|^2-|x|^2\big)\mathrm{Id}_m\Big]\pi(dx\:dy)
\]so that \begin{equation}\label{eqn:psi_decomp}\psi_\pi(x)=|x|^2\mathrm{Id_m}-\sum_{j=1}^m \eta_j\eta_j^T\end{equation}. For simplicity, assume that $|\eta_1|<|\eta_2|<\cdots <|\eta_m|$. Evaluating the marginal problem along the lines $r_i(t)=t \tfrac{\eta_i}{|\eta_i|}$ one finds\[
\frac{d}{dt}J_\pi(r_i(t)|x)=t^3-\big(|x|^2-|\eta_i|^2 \big) t-\frac{\varphi_\pi^T(x)\eta_i}{|\eta_i|},
\]which can have multiple solutions along $r_i(t)$ provided $|x|>|\eta_i|$.

This alone is not necessarily a problem under Theorem \ref{marginal_problem} in that the marginal problem may have several critical points, but as long as there is a unique global minimizer we may still guarantee existence of non-probabilistic solutions for dimension reduction problem. This said, consider the set $\{x\,|\,\varphi_\pi(x)=0\}$ where the critical point equation can be expressed as
\[
\Big(|y|^2-|x|^2\Big)y+\sum_{j=1}^m \eta_j\eta_j^Ty=0.
\]One may readily check that the solutions to the above equation are exhausted by $y=\pm \tfrac{\eta_j}{|\eta_j|}\sqrt{|x|^2-|\eta_j|^2}$ for $j=1,...,m$ and $y=0$. The previous observations imply that the former case is only possible when $|x|>|\eta_j|$ which makes the square root well defined.
Plugging in each of these critical points into the marginal problem, we find that\[
J_\pi\left(\pm\tfrac{\eta_j}{|\eta_j|}\sqrt{|x|^2-|\eta_j|^2}\,\Big|\,x\right)=-\big(|x|^2-|\eta_j|^2\big)^2+\zeta_\pi(x)\geq -\big(|x|^2-|\eta_i|^2\big)^2+\zeta_\pi(x).
\]where $i$ is the largest index for which $|x|\leq |\eta_{i+1}|$. Hence for $\{x\,|\, \varphi_\pi(x)=0,\,|\eta_{i}|<|x|\leq |\eta_{i+1}|\}$, there are two minimizers to the marginal problem: $\pm \tfrac{\eta_i}{|\eta_i|}\sqrt{|x|^2-|\eta_i|^2}$. The case devolves further if $|\eta_j|$ is repeated ($|\eta_1|<\cdots<|\eta_j|=\cdots =|\eta_{j+k-1}|<\cdots<|\eta_m|$) and $|x|\leq |\eta_{j+1}|$ where any $y$ on the $k$-sphere spanned by $\eta_j,...,\eta_{j+k-1}$ is a minimizer of $J_\pi(\cdot |x)$.
\end{exmp}
The previous example is meant to demonstrate how pathological the nature of the marginal minimization problem can be: \emph{for simple costs, the marginal minimizers may be comprised of entire sub-manifolds in $\R^m$ for a single $x$!} In the pursuit of deterministic minimizers (i.e. Monge-type maps), one approach might be to show that these multiple minimizers can only happen on a thin set (in the above example this corresponds to showing that $\varphi_\pi(x)$ is full rank $\mu(dx)$-a.e.), but due to the implicit dependence of the marginal problem on the embedding plan $\pi$, taking this route directly has proven particularly difficult. 

 Another notable consequence which can be observed from the marginal problem framework is that for normed costs, it will be likely that there will be jump discontinuities arising from an analogous phenomenon to that of Example \ref{exmp:double_well}. The following example shows that in the case of q-MDS, we can guarantee discontinuities in the optimal solution. We can expect the argument below to persist for any dimension reduction problem for which $\arg\min J_\pi(\cdot|x)$ has multiple values for some $x$, but this property is implicitly dependent on $\pi$ as well and thus challenging to verify in practice.
\begin{exmp}
\label{exmp:discontinuous}
Putting technicalities of the rank of $D_y J_\pi(y|x)$ aside for the moment,  Example \ref{qMDS_MP} in the previous section argues that when none of the lengths of $|\eta_j|$ are repeated, there are $m+1$ distinct regions for which the marginal problem is defined by a different solution. More precisely, for $A_i:=\{x\,|\,\varphi_\pi(x)=0, |\eta_i|<|x|\leq |\eta_{i+1}|\}$, we have a semi-explicit (governed by moments of the optimal solution) formula for the reduction map: $T(x)=\pm\tfrac{\eta_i}{|\eta_i|}\sqrt{|x|^2-|\eta_i|^2}$; when $x$ passes from $A_i$ to $A_{i+1}$ the optimal solution abruptly jumps from $\pm\tfrac{\eta_i}{|\eta_i|}\sqrt{|x|^2-|\eta_i|^2}$ to $\pm\tfrac{\eta_{i+1}}{|\eta_{i+1}|}\sqrt{|x|^2-|\eta_{i+1}|^2}$.

Beyond this, one also can observe that for any path $x+\varepsilon v$ with $\varphi_\pi(x)=0$ (and $v$ not in the nullspace of $\Phi_\pi$) the marginal minimizer has a jump discontinuity at $\e = 0$. The intuition here will come from Example \ref{exmp:double_well}. Indeed, the previous considerations have implied that there will be multiple minimizers when $\varphi(x)=0$, namely $\pm\tfrac{\eta_i}{|\eta_i|}\sqrt{|x|^2-|\eta_i|^2}$ where $i$ is the smallest index such that $|x|> |\eta_{i}|$. By plugging $\pm\tfrac{\eta_i}{|\eta_i|}\sqrt{|x|^2-|\eta_i|^2}$ into the marginal problem \eqref{qMDS_CPE_cont} at $x+\varepsilon v$. We see by \eqref{eqn:psi_decomp},
\begin{align*}
J_\pi\big(\pm\tfrac{\eta_i}{|\eta_i|}\sqrt{|x|^2-|\eta_i|^2}\,\big|\, x+\varepsilon v\big)&=\big(|x|^2-|\eta_i|^2\big)^2\\
&-2\tfrac{\eta_i^T}{|\eta_i|}\sqrt{|x|^2-|\eta_i|^2}\Big(|x+\varepsilon v|^2\mathrm{Id}_m-\sum_{j=1}^m\eta_j\eta_j^T\Big)\tfrac{\eta_i}{|\eta_i|}\sqrt{|x|^2-|\eta_i|^2}\\
&\pm4\varphi_\pi^T(x+\varepsilon v)\frac{\eta_i}{|\eta_i|} \sqrt{|x|^2-|\eta_i|^2}+\zeta_\pi(x+\varepsilon v)\\
&=\big(|x|^2-|\eta_i|^2\big)^2-2\Big(|x|^2-|\eta_i|^2\Big)\Big(|x+\e v|^2-|\eta_i|^2\Big)\\
&\pm8\varepsilon v^T\Phi_\pi^T \frac{\eta_i}{|\eta_i|} \sqrt{|x|^2-|\eta_i|^2}+\zeta_\pi(x+\varepsilon v)
\end{align*}
where on the last line we have used the fact that $\varphi_\pi(x+\varepsilon v)=2\varepsilon \Phi_\pi v $. Crucially, we see that in order for the above expression to be minimal, one needs to choose the sign of the order $\varepsilon$ term to be opposite that of $v^T\Phi_\pi^T\eta_i$. In particular, this shows that near a point $x$ for which $\varphi_\pi(x)=0$, the optimal map is\[
T(x+\varepsilon v)=-\mathrm{sign}(\varepsilon v^T\Phi_\pi^T\eta_i)\tfrac{\eta_i}{|\eta_i|}\sqrt{|x|^2-|\eta_i|^2}+\O(\varepsilon)
\]whose limit does not exist at $\varepsilon=0$.

 \end{exmp}
 
 Now having seen the possibility of multiple minimizers to the marginal problem and how it can cause discontinuities, we illustrate one more useful perspective in the context of dimension reduction. Being that dimension reduction schemes inherently discard information while representing data in the embedded space, there must be some partition of $\X$ such that each element of the partition may be represented by a single value in the embedding. More precisely, for the map outlined in Definition \ref{defn:marginal_problem}, the set $\{x:\lambda(x)=y\}$ represents all of the points in $\X$ which are optimally embedded to the vector $y$. While these sets can be arbitrary, we expect them to form $d-m$ dimensional manifolds. To illustrate this, we present one more example.
 \begin{exmp}
 \label{exmp:equivilance_classes}Let us consider a simple example where 1000 datapoints in $\R^2$ are such that 500 points are stacked at $(0,1)$ and the other 500 are stacked at $(0,-1)$. The optimal embedding for the q-MDS cost into $\R$ is clearly realized by projecting the 2 dimensional dataset onto the $y$-axis. This allows us to explicitly compute\[
 \psi_\pi(x_1,x_2)=x_1^2+x_2^2-2,\qquad\varphi_\pi(x_1,x_2)=2x_2
 \]thus the critical point equation can be written $y^3-(x_1^2+x_2^2-2)y=2x_2$. Imitating the previous computations, we first notice that when $|x|<\sqrt{2}$, $\psi(x_1,x_2)<0$. This implies that on disk of radius $\sqrt{2}$, the marginal problem \eqref{qMDS_CPE_cont} has a unique solution. Indeed for $x$ in the set $\{x:|x|<\sqrt{2}\}$,\[
 \frac{d^2}{dy^2}J_\pi(y|x)=12 y^2-4\psi_\pi(x)>0.
 \]Furthermore, if $|x|\geq \sqrt{2}$, there are be multiple minimizers along the set $\{x:\varphi_\pi(x)=0\}$ which will lead to a jump discontinuity as predicted in Example \ref{exmp:discontinuous}. The figure below illustrates the level sets of the minimizers, $\lambda_\pi:\R^2\to\R$.

 \begin{figure}[h]
    \centering

     \includegraphics[width=7cm]{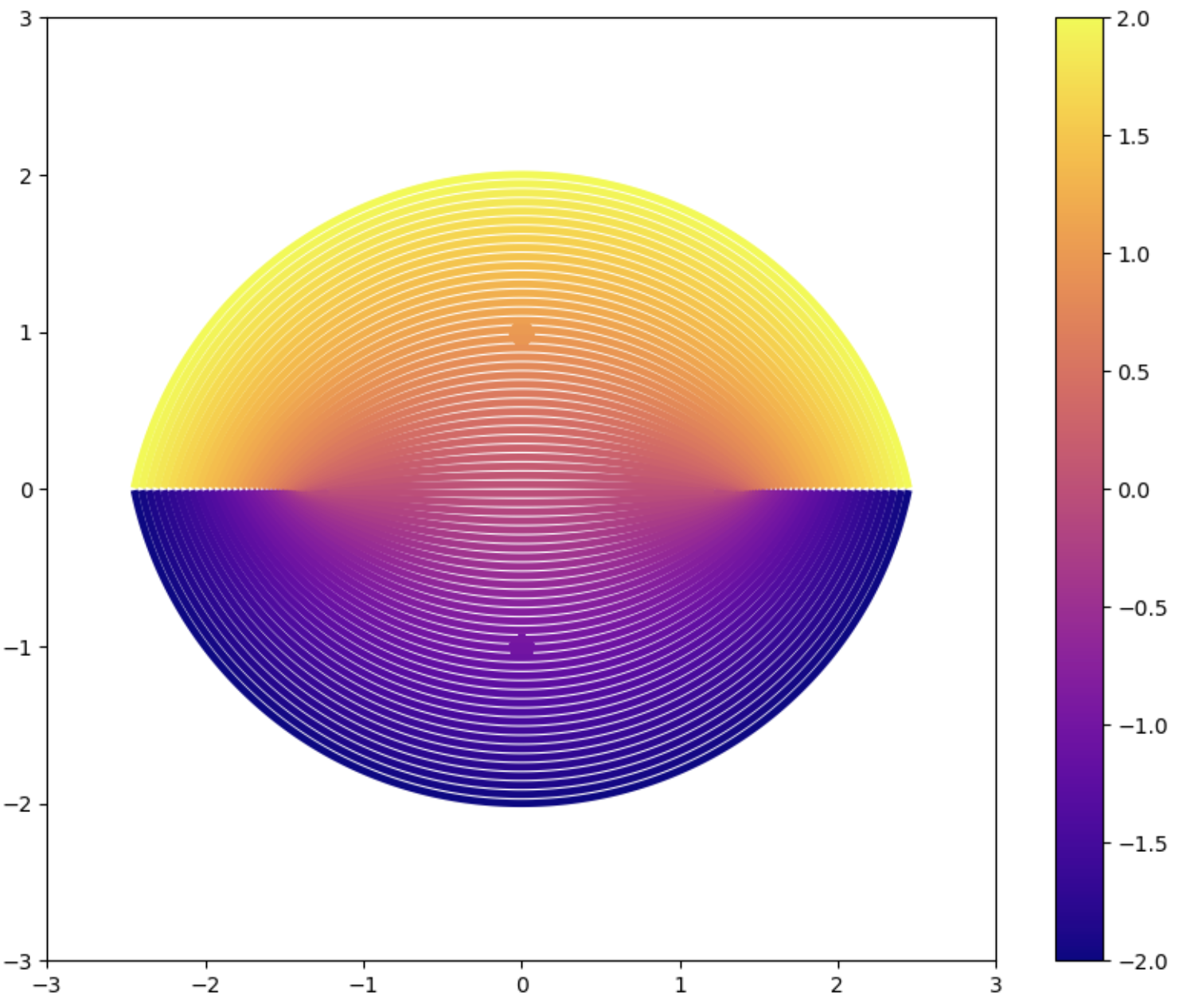}
     \caption{Each band represents an equivalence class of points in $\R^2$ which all have the same minimizer in $\R$ for the embedding outlined in Example \ref{exmp:equivilance_classes}. Notice that once $|x|>\sqrt{2}$, the line $x_2=0$ has a discontinuity surface.}
     \label{fig:equiv_classes}
\end{figure}
 \end{exmp}

\section{Normed costs: maps via second-order conditions}
\label{sec:Normed_costs}

In this section, we show that for a wide range of normed costs the solution of the dimension reduction problem \ref{DR} is induced by a map. The main difficulty is that the dimension reduction problem of type $\mathbf{(N}^2)$ is not marginally convex in $y$ (i.e. $y\mapsto J_\pi(y|x)$ is not convex) and thus we expect multiple minimizers to a given marginal problem (take for instance Example \ref{exmp:numerical}). This follows since $y\mapsto c(x,x',|y-y'|^2)$ need not be convex  \emph{even when the function} $t\mapsto c(x,x',t)$ \emph{is convex}. To surmount this, we track the second-order effect of perturbations, on the level of the dimension reduction plans. We shall see that certain natural structural conditions upon $c$ (namely assumptions $\mathbf{(A3)}$-$\mathbf{(A5)}$) will then be sufficient to guarantee that optimal plans are induced by maps. 

As in Section \ref{sec:marginal}, we motivate our proofs by first formally considering the case where the input distribution is realized as a sum of Dirac masses, $\mu=(1/n)\sum_{i=1}^n \delta_{x_i}$, for some collection of distinct points $x_1,x_2,...,x_n\in\R^d$. As in Section \ref{sec:marginal}, we assume that the optimal embedding may be represented discretely by $\pi =(1/n)\sum_{ij}\pi_{ij}\delta_{(x_i,y_j)}$ for some distinct collection of vectors $y_1,y_2,...,y_n\in\R^m$. Suppose that in the $i$th row of $\pi$ there are at least two nonzero entries and reorder the $y$'s so that $\pi_{ii}, \pi_{ij}>0$; this essentially encodes the situation where an optimal embedding maps a single $x$ to multiple $y$'s. 

By Theorem \ref{marginal_problem}, both $y_i,y_j\in\arg\min J_\pi(\cdot|x_i)$ and thus we can transport the mass stored at $(x_i,y_j)$ to $(x_i,y_i)$ without violating our first-order condition. More precisely, if $\pi_{ii} >0$ and $\pi_{ij}>0$, the perturbation\[
\dpi=\min\{\pi_{ii},\pi_{ij}\}\big[\delta_{(x_i,y_i)}-\delta_{(x_i,y_j)}\big],
\]is well-defined and will have $\J(\dpi|\pi)=0$, meaning that it will leave the energy unchanged up to second-order variations. When we compute the quadratic term, we have\[
\dpi\otimes \dpi = \delta_{(x_i,y_i)}\otimes \delta_{(x_i,y_i)}-\delta_{(x_i,y_i)}\otimes \delta_{(x_i,y_j)}-\delta_{(x_i,y_j)}\otimes \delta_{(x_i,y_i)}+\delta_{(x_i,y_j)}\delta_{(x_i,y_j)}
\]and thus\begin{align*}
\J(\dpi)&=c(x_i,x_i,|y_i-y_i|^2)-2c(x_i,x_i,|y_i-y_j|^2)+c(x_i,x_i,|y_j-y_j|^2)\\
&= 2\big[ c(x_i,x_i,0)- c(x_i,x_i,|y_i-y_j|^2)\big].
\end{align*}Crucially, if $t\mapsto \tilde c(x,x,t)$ has a strict global minimum for $t=0$ for all $x$, then $\J(\dpi)<0\iff y_i\ne y_j$. This implies that for each $x$ the optimal plan must be supported only on a single $y$. We remark that this argument works for any type of dissimilarity kernel. In spite of the technical difficulties engendered by the loss of lower semicontinuity, this discrete argument suggests a very strong result: that solutions to the relaxed problem \eqref{DR} with normed cost \emph{must be deterministic}, and hence must solve the original problem \eqref{eqn:MDS-general}.  This is quite surprising in light of the examples presented in the introduction suggesting that Young measures can be encountered in practice. This is because the perturbations used in particle-based optimization methods cannot carry out perturbations of the form $\dpi$, and can get stuck in local mins with respect to particle-wise descent.

To extend this argument to the continuum setting, one must be able to represent solutions to the marginal problem locally in a consistent manner. In particular, it would be ideal to obtain more structure on the nature of the marginally minimizing set-valued map $\lambda$ as outlined in Definition \ref{defn:marginal_problem}. Leaving technical justification aside for the moment, suppose that locally $\lambda$ admits a countable representation, i.e. $\lambda(x)=\bigcup_{i=1}^\infty \lambda_i(x)$ for a sequence of smooth functions $\lambda_i$. We then can leverage the discrete argument between pairs of these functions through the following proposition.

\begin{prop}\label{prop:paired-support}
  Let $\lambda_1,\lambda_2:B_\delta(x_0) \to Y$ be continuous functions with $\lambda_1(x) \neq \lambda_2(x)$ for all $x \in B_\delta(x_0)$. Assume that $c$ is a continuous cost of type $\mathbf{(N}^2)$ and satisfies assumptions $\mathbf{(A1)}$-$\mathbf{(A3)}$ and assume $\pi$ be a minimizer of \eqref{DR}. Let $\mu_1,\mu_2$ be the $\X$-marginal measures of $\pi$ restricted to the sets $y = \lambda_1(x)$ and $y=\lambda_2(x)$ and $x \in B_\delta(x_0)$. Then $\mu_1$ is mutually singular to $\mu_2$, or in symbols $\mu_1 \perp \mu_2$, meaning that they have disjoint supports.
\end{prop}
\begin{proof}
  Suppose, for the sake of contradiction, that $\mu_1$ and $\mu_2$ are not mutually singular. Then the measure $\mu_1 \wedge \mu_2 = \mu_1 - (\mu_1-\mu_2)^+$ is not a zero measure, and we may select a point $\bar x \in B_\delta(x_0)$ so that $\mu_1 \wedge \mu_2(B_\e(\bar x)) >0$ for all $\e>0$ sufficiently small.

  We then construct the perturbation, restricted to $x \in B_\e(\bar x)$, via \[
  \dpi(dx\:dy)= \mu_1 \wedge \mu_2 (dx) \Big[\delta_{\lambda_{1}(x)}(dy)-\delta_{\lambda_2(x)}(dy)\Big].
\] By construction we have that $\pi + \dpi$ is a probability measure and retains the same $\X$ marginal as $\pi$. We also note that $\dpi$ is not the zero measure by choice of $\bar x$.

As long as $\mu_1,\mu_2$ are non-trivial, then by Theorem \ref{marginal_problem} we know that $\lambda_1,\lambda_2$ must be minimizers of the marginal problem on the support of $\mu_1,\mu_2$. Using the notation from the proof of Theorem \ref{marginal_problem} we have that
\[
  \J(\dpi|\pi)=\int J_\pi(y|x)\dpi(dx\:dy)=\int_{B_\varepsilon(\bar{x})} [\min J_\pi(\cdot |x)-\min J_\pi(\cdot |x)]\mu_1 \wedge \mu_2(dx)=0.
\]
The overall change in the quadratic term is given by, after using \textbf{(A2)},\begin{align*}
\J(\dpi)&=\iint_{B_\varepsilon(\bar x)\times B_\varepsilon(\bar x)} c(x,x',|\lambda_i(x)-\lambda_i(x')|^2)\mu_1 \wedge \mu_2(dx)\mu_1 \wedge \mu_2(dx')\\
&+\iint_{B_\varepsilon(\bar x)\times B_\varepsilon(\bar x)}  c(x,x',|\lambda_j(x)-\lambda_j(x')|^2)\mu_1 \wedge \mu_2(dx)\mu_1 \wedge \mu_2(dx')\\
&-2\iint_{B_\varepsilon(\bar x)\times B_\varepsilon(\bar x)}   c(x,x',|\lambda_i(x)-\lambda_j(x')|^2)\mu_1 \wedge \mu_2(dx)\mu_1 \wedge \mu_2(dx').
\end{align*}
By using the continuity of $c,\lambda_1,\lambda_2$, we then estimate
\[
  \J(\dpi) \leq 2 \mu_1 \wedge \mu_2(B_\e(\bar x))^2 (c(\bar x,\bar x,0) - c(\bar x,\bar x, |\lambda_1(\bar x)- \lambda_2(\bar x)|^2) + \eta(\e)),
\]
where $\eta$ represents a local modulus of continuity and satisfies $\eta(\e) \to 0$ as $\e \to 0$. As $\lambda_1(\bar x) \neq \lambda_2(\bar x)$, and $c(x,x,t)$ is strictly minimized at $t=0$ by \textbf{(A3)}, we obtain that $\J(\dpi) < 0$, which contradicts the minimality of $\pi$.
\end{proof}
\noindent An induction argument then gives the following immediate corollary.
\begin{corollary}
  Let $\lambda_i:O_i \to Y$ be continuous functions, where $O_i$ are open sets, and $i \in \{1, \dots, \infty\}$. Let $\pi$ be a minimizer of \ref{DR} for continuous cost satisfying $\mathbf{(N}^2)$ and $\mathbf{(A1)}$-$\mathbf{(A3)}$, and let $\tilde \pi$ be the restriction of $\pi$ to the union of the sets $\{(x,\lambda_i(x)) : x \in O_i\}$. Then $\tilde \pi$ has support on the graph of a function.
  \label{cor:countable-graph}
\end{corollary}

The previous proposition offers a direct application to global minimizers of the marginal problem which have non-degenerate Hessian in $y$; namely those minimizers which are also strict local minimizers. We begin by proving two brief lemmas based upon the implicit function theorem.

\begin{lemma}
  Let $\pi$ be a minimizer of \ref{DR} for cost satisfying $\mathbf{(N}^2)$ and $\mathbf{(A1)}$-$\mathbf{(A4)}$. Suppose that $y_1 \neq y_2$ are global minimizers of the marginal problem at $\bar x$, which both satisfy $D_{yy}^2 J_\pi(y_i | \bar x) > 0$. Then there exists a $\delta>0$ and $C^1$ functions $\lambda_i:B_\delta(\bar x) \to B_\delta(y_i)$, $i=1,2$ so that $\lambda_i(x)$ is the only strict local minimizer of the marginal problem in $B_\delta(y_i)$. 
  \label{lem:Implicit-solutions}
\end{lemma}
\begin{proof}
  The minimality of $y_1$ and $y_2$ indicate that both $D_y J_\pi(y_1|\bar x)=0$ and $D_y J_\pi(y_2|\bar x)=0$. From the strict non-degeneracy assumption on $D_y^2 J_\pi$, the implicit function theorem allows us to construct $C^1$ maps $\lambda_i:B_\delta(\bar x)\to B_\delta(y_i)$ $i=1,2$, which uniquely solve $D_y J_\pi(\lambda_i(x)|x)=0$ on the respective neighborhoods in the product space. We note that without loss of generality $\delta$ can be taken small enough to guarantee the strict local minimality of $\lambda_1$ and $\lambda_2$ since $J_\pi$ was $C^2$. 
\end{proof}

\begin{lemma}
  Assume that $c$ is a cost of type $\mathbf{(N}^2)$ and satisfies assumptions $\mathbf{(A1)}$-$\mathbf{(A4)}$ and let $\pi$ be a minimizer of \eqref{DR}. Then for every $x$ there exists at most a countable number of global minimizers of the marginal problem which satisfy $D_{yy}^2 J_\pi(y|x) > 0$.
  \label{lem:countable-mins}
\end{lemma}

\begin{proof}
  First note that since $c$ is $C^2$ by \textbf{(A4)}, it follows from Lemma \ref{lem:marginal-problem-smooth} that the marginal problem is a $C^2$ function in $y$. Furthermore, by \textbf{(A1)} the minimizers of the marginal point at a point $x$ must live in a compact set $K_x \subset \Y$. Consider the set $M_\eta \subset K_x$ of global minimizers of the marginal problem at $x$ satisfying $|D^2_y J_\pi(y|x)|\geq \eta$. We notice that $M_\eta$ will also be compact. As $J_\pi(y|x)$ is $C^2$, each element of $M_\eta$ can be surrounded by a ball of some radius $r_\eta>0$ which contain no other point in $M_\eta$: this essentially says that a global minimizer with a lower bound on the Hessian is an isolated minimizer with a quantifiable distance of isolation. As $M_\eta$ is compact, we then have that it actually must be finite. By taking $\eta$ to zero, this argument shows that the number of minimizers with non-degenerate Hessian must be at most countable.
\end{proof}

We now choose to decompose the optimal plan into points where the Hessian is non-degenerate (i.e. rank strictly less than $m$) and its complement via\begin{equation}\label{eqn:pi-decomposition}
  \pi = \pi_S + \pi_I,\qquad \pi_S = \pi_{|\mathrm{det}(D_{yy}^2 J_\pi) = 0}, \qquad \pi_I = \pi_{|\mathrm{det}(D_{yy}^2 J_\pi) \neq 0}.
\end{equation}
In terms of this decomposition, we can use Corollary \ref{cor:countable-graph} along with Lemma \ref{lem:countable-mins} to immediately give the following.

\begin{prop}\label{prop:pi-non-degenerate-graph-function}
  Let $\pi$ be a minimizer of \eqref{DR}, for $c$ of type $(\mathbf{N}^2)$ and satisfies assumptions $\mathbf{(A1)}$-$\mathbf{(A4)}$. Using the decomposition \eqref{eqn:pi-decomposition}, then $\pi_I$ is supported on the graph of a function.
\end{prop}

The only remaining point is to rule out multivaluedness at points where the Hessian of the marginal problem is degenerate. We address this issue completely in the following proposition.
\begin{prop}
Assume $c$ is a cost of type $(\mathbf{N}^2)$ and satisfies assumptions $(\mathbf{A1}),\,(\mathbf{A2}),$ and $(\mathbf{A5})$. If $\pi$ is an optimal plan, then $\pi_S$ is concentrated on the graph of a function.
  \label{prop:degenerate-region-map}
\end{prop}

\begin{remark}
  In this theorem we notice that there are no requirements on the measure $\mu$, nor on $m,d$. Furthermore, we notice that in the statement we can say that $\pi$ is induced by a map on the set where $\rho_s>0$, and not just $\pi_S$. Hence any part of the support not covered by Proposition \ref{prop:degenerate-region-map} will be covered by Proposition \ref{prop:pi-non-degenerate-graph-function}.
\end{remark}

\begin{proof}
  The main idea of the proof lies in tracking second variations along smooth perturbations of $y$. A portion of these perturbations are chosen to be in directions where the marginal problem is, up to second-order, degenerate, so that the $y,y'$ terms in the second-order Taylor expansion dominate.

  To begin, let $E := \{(x,y) : \mathrm{det}(D_{yy}^2 J_\pi(y|x)) = 0\}$. We then choose a measurable function $\phi_0: E \to \mathbb{S}^{m-1}$ such that
  \[
    D_{yy}^2 J_\pi(y|x)\cdot \phi_0(x,y)=0.
  \]
  The existence of such a function can be justified using measurable selections of the multifunction encoding the nullspace of $D_{yy}^2 J_\pi(y|x)$, see for example \cite{rockafellar2009variational}.

  We will consider a point $x_0$ at which $\rho_s > 0$ in the sense of Lebesgue points. Select a unit vector $v$ so that 
  \[
    \liminf_{\delta \to 0^+} \frac{\pi_S(E \cap B_\delta(x_0)\times \Y \cap \{\phi_0(x,y) \cdot v \geq 1/2\})}{\pi_S(E \cap B_\delta(x_0)\times \Y)} =: \underline{\rho} >0.
\]Such a vector must exist because we can cover the unit sphere with a finite number of cones with opening angle $2\pi/3$, and we have assumed that $\rho_s(x_0) > 0$. For any choice of $\delta>0$, we write
\[
  E_v := E \cap B_\delta(x_0) \times \Y \cap \{\phi_0(x,y) \cdot v \geq \cos(\pi/8)\}
\]
we notice that for $(x_1,y_1),(x_2,y_2) \in E_v$ we have that $\phi(x_1,y_1) \cdot \phi(x_2,y_2) \geq \frac{\sqrt{2}}{2}$.

We note that there exists $K$ so that $|D_{yy}^2 J_\pi(y|x)| < K$ for $x \in B_\delta(x_0)$ and $y \in \arg\min J_\pi(\cdot|x)$. Such a $K$ exists because of the $C^2$ bounds on the marginal problem and the locally uniform compactness of the minimizers of the marginal problem.

Now we define a function $\phi: \X \times \Y \to \Y$ by
\begin{displaymath}
  \phi(x,y) := \begin{cases} \phi_0(x,y) &\text{if } (x,y) \in E_v \\
    \beta v &\text{if } x \in B_\delta(x_0) \text{ and } (x,y) \notin E_v, \\
    0 &\text{otherwise,}
  \end{cases}
\end{displaymath}
where $\beta>0$ is a parameter that we will select later. We now utilize this function $\phi$ to construct a one parameter family of functions
%
$\varphi_x^\varepsilon(y):=y+\varepsilon \phi(x,y)$ and an associated family of plans 
  $\pi_\e$ by writing\[
\pi_\e(dx\:dy)=\varphi_x^\varepsilon\sharp \nu(dy|x)\mu(dx)
\]where $\pi(dx\:dy)=\nu(dy|x)\mu(dx)$ by disintegration. 
We then compute
\[
  \J(\pi_\e) - \J(\pi) = \iint c(x,x',y + \varphi_x^\varepsilon(y),y' + \varphi_{x'}^\varepsilon(y')) - c(x,x',y,y') \pi(dx\:dy) \pi(dx'dy').
\]
Taylor expanding $c$ we then obtain
\begin{align*}
  \J(\pi_\e) - \J(\pi) &= \iint D_y c(x,x',y,y')\e \phi(x,y) + D_{y'} c(x,x',y,y') \e \phi(x',y') + 1/2 \e^2\phi^T(x,y) D_{yy}^2 c(x,x',y,y') \phi(x,y) \\
  &+ 1/2 \e^2\phi^T(x',y') D_{y'y'}^2 c(x,x',y,y') \phi(x',y') + \e^2\phi^T(x,y) D_{yy'}^2 c(x,x',y,y') \phi(x',y') \pi(dx\:dy) \pi(dx'dy')\\
  &+ \O(\e^3).
\end{align*}
By using Fubini's theorem along with Theorem \ref{marginal_problem} and Equation \eqref{gen_CPE}, we immediately have that the order $\e$ terms vanish.

The order $\e^2$ terms take the form, after removing sets where $\phi = 0$, using the fact that $\phi$ is in the nullspace of $D_{yy}^2 c$ on $E_v$, and assuming that $\beta$ is sufficiently small
\begin{align*}
  &\iint_{E_v \times E_v}  \phi^T(x,y) D_{yy'}^2 c(x,x',y,y') \phi(x',y')\pi(dx\:dy) \pi(dx'dy') \\
&+ \iint_{E_v \times B_\delta(x_0)\times \Y \setminus E_v} \beta v^T D_{y y'}^2 c(x,x',y,y') \phi(x',y')\pi(dx\:dy) \pi(dx'dy') \\
&+ \iint_{B_\delta(x_0)\times \Y \setminus E_v \times B_\delta(x_0)\times \Y \setminus E_v} \beta^2 v^T D_{yy}^2 c(x,x',y,y') v \pi(dx\:dy) \pi(dx'dy')\\
&\leq  \frac{\sqrt{2}}{2} \iint_{E_v \times E_v} -\psi_1(|y-y'|^2)+\psi_2(|x-x'|^2) \pi(dx\:dy) \pi(dx'dy') \\
&+ \beta \cos(\pi/8) \iint_{E_v \times B_\delta(x_0)\times \Y \setminus E_v } -\psi_1(|y-y'|^2)+\psi_2(|x-x'|^2) \pi(dx\:dy) \pi(dx'dy') + \beta^2 K \pi(B_\delta(x_0)\times \Y \setminus E_v)^2.
\end{align*}
By taking $\beta$ sufficiently small, we can neglect the last term. Minimality implies that this entire quantity must by $\geq 0$, and hence by taking $\delta \to 0$ and using the fact that the $\psi_1,\psi_2$ are strictly increasing and zero at zero immediately implies that $\nu_{x_0}$ must be given by a Dirac mass. This then implies that on the set where $\rho_s>0$ we have that $\pi$ is supported on the graph of a function.
\end{proof}

Together we have now proven our main theorem, which for simplicity is presented with assumption $\mathbf{(A0)}$ (which implies assumption $\mathbf{(A3)})$ which is the requirement for normed costs) so to simultaneously include normed and inner product costs. 
\begin{theorem}[Deterministic solutions a.k.a. Monge Maps]\label{DP} Let $\mu\in\Pp(\R^d)$ with cost structure either $\mathbf{(IP)}$ or $\mathbf{(N}^2)$ and satisfies assumptions $\mathbf{(A0)}$-$\mathbf{(A5)}$. Then solutions to $\min_{\pi\in\Pi(\mu)} \J(\pi)$ are concentrated on the graph of a function; i.e. there is a measurable $T:\R^d\to\R^m$ such that $\pi(dy|x)=\delta_{T(x)}(dy)$, $\mu(dx)$-almost everywhere. More succinctly, solutions to the dimension reduction problem exist and are necessarily deterministic.
\end{theorem}
\noindent It is the necessity of a deterministic solution which is surprising in view of Example \ref{exmp:numerical} and compliments the work of \cite{article2}.

\section{Conclusion}
\label{sec:conclusion}
In this work we have examined theoretical properties of some fundamental dimension reduction algorithms. In doing so, we have focused on the optimization problem and necessary conditions associated with population level problems. We have shown that, for natural costs based upon similarities (i.e. inner products), and dissimilarities (i.e. norm differences), that the dimension reduction problem must be minimized by a deterministic mapping, and that any probabilistic behavior is necessarily sub-optimal.

On the other hand, the behavior that we observe in Example \ref{exmp:numerical} raises many difficult questions. Clearly local minimizers found using naive particle descent methods may exhibit probabilistic behavior, which is consistent with the failure of lower-semicontinuity we proved in Proposition \ref{DR_lacks_weak_LSC}. On the level of practical applicability, we find such probabilistic behavior highly problematic. For example, it could lead to very misleading clustering in data visualization, where similar points in feature space are probabilistically assigned to distinct clusters.

These issues raise many natural follow up questions, a few of which we list here:

\begin{itemize}
  \item Are the issues with probabilistic minimizers found via particle descent methods still present in real-world data sets? We have not pursued this issue here because comprehensively addressing this question calls for a detailed study across numerous benchmark data sets.
  \item What computational methods can be developed to avoid spurious probabilistic behavior in dimension reduction, and how can the necessary conditions identified in this work be used to do so?
  \item If non-linear dimension reduction algorithms often induce discontinuous embeddings, how greatly can they modify the topology of the data in feature space?
  \item Is similar behavior relevant in other unsupervised learning methods?
\end{itemize}

We hope that these questions help to stimulate a more detailed study of dimension reduction methods.

\section*{Acknowledgements}

The authors gratefully acknowledge the support of NSF DMS 2307971 and the Simons Foundation MP-TSM. AP also warmly thanks Peter McGrath and Erik Bates for many helpful discussions about early versions of the work.
\printbibliography
\end{document}